\documentclass[twoside,11pt]{article}

\usepackage[T1]{fontenc}
\usepackage[utf8]{inputenc}
\usepackage{blindtext}
\usepackage{ifthen}

\IfFileExists{jmlr2e.sty}{%
  \usepackage[abbrvbib,preprint]{jmlr2e}%
}{%
  \usepackage[numbers]{natbib}%
  \usepackage{amssymb,graphicx}%
  \newcommand{\jmlrheading}[7]{}%
  \newcommand{\ShortHeadings}[2]{}%
  \newcommand{\firstpageno}[1]{\setcounter{page}{##1}}%
  \newcommand{\email}[1]{\texttt{##1}}%
  \newcommand{\addr}[1]{##1}%
  \newcommand{\name}[1]{##1}%
  \newcommand{\AND}{\and}%
  \newcommand{\editor}[1]{}%
  \newtheorem{theorem}{Theorem}%
  \newenvironment{proof}{\par\noindent\textbf{Proof.}}{\hfill$\square$\par}%
  \newenvironment{keywords}{\par\noindent\textbf{Keywords: }}{\par}%
}

\usepackage{amsmath,amsfonts}
\usepackage{algorithm}
\usepackage{algorithmic}
\usepackage{array}
\usepackage{textcomp}
\usepackage{stfloats}
\usepackage{url}
\usepackage{verbatim}
\usepackage[caption=false,font=footnotesize]{subfig}
\usepackage{tabularx}
\usepackage{booktabs}
\usepackage{diagbox}
\usepackage{multirow}
\usepackage{lastpage}

\makeatletter
\@ifundefined{assumption}{}{}
\makeatother

\newcommand{\missingfigure}[1]{%
  \fbox{\parbox[c][0.20\textheight][c]{0.85\linewidth}{\centering Missing figure file:\\ \texttt{#1}}}%
}
\newcommand{\safeincludegraphics}[2][]{%
  \IfFileExists{#2}{\includegraphics[#1]{#2}}{\missingfigure{#2}}%
}
\newcommand{\safeinput}[1]{%
  \IfFileExists{#1}{\input{#1}}{\typeout{Warning: missing input file #1, skipped.}}%
}

\jmlrheading{XX}{2026}{1-\pageref{LastPage}}{Submitted 3/26}{Under review}{XX-0000}{Xuan Qi, Yi Wei, Fanqi Yu, Furao shen, Vittorio Murino, and Cigdem Beyan}

\ShortHeadings{AffineLens}{Wei, Qi, Shen, Zhao, Murino, and Beyan}
\firstpageno{1}

\begin{document}

\title{AffineLens: Capturing the Continuous Piecewise Affine Functions of Neural Networks}

\author{\name Yi Wei$^{*}$ \email ywei@smail.nju.edu.cn \\
       \addr State Key Laboratory of Novel Software Technology\\
       School of Intelligence Science and Technology\\
       Nanjing University, Jiangsu, China\\
       \AND
       \name Xuan Qi$^{*}$ \email xuan.qi@iit.it \\
       \addr AI for Good\\
       Istituto Italiano di Tecnologia, Genoa, Italy\\
       DITEN\\
       University of Genoa, Genoa, Italy\\
       \AND
       \name Furao Shen$^{\dagger}$ \email frshen@nju.edu.cn \\
       \name Jian Zhao \email jianzhao@nju.edu.cn \\
       \addr State Key Laboratory of Novel Software Technology\\
       School of Artificial Intelligence\\
       Nanjing University, Jiangsu, China\\
       \AND
       \name Vittorio Murino \email vittorio.murino@iit.it \\
       \name Cigdem Beyan \email cigdem.beyan@univr.it \\
       \addr AI for Good\\
       Istituto Italiano di Tecnologia, Genoa, Italy\\
       Department of Computer Science\\
       University of Verona, Verona, Italy}
\editor{My editor}

\maketitle
\begingroup
\renewcommand{\thefootnote}{\fnsymbol{footnote}}
\footnotetext[1]{Equal contribution. \quad $^{\dagger}$Corresponding author.}
\endgroup

\begin{abstract}
Piecewise affine neural networks (PANNs) provide a principled geometric perspective on neural network expressivity by characterizing the input--output map as a continuous piecewise affine (CPA) function whose complexity is governed by the number, arrangement, and shapes of its affine regions. However, existing interpretability and expressivity analyses often rely on indirect proxies (e.g., activation statistics or theoretical upper bounds) and rarely offer practical, accurate tools for enumerating and visualizing the induced region partition under realistic architectures and bounded input domains. In this work, we present AffineLens, a unified framework for computing the hyperplane arrangements and polyhedral structures underlying PANNs. Given a calibrated (bounded) input polytope, AffineLens identifies the subset of neuron-induced hyperplanes that intersect the domain, enumerates the resulting affine sub-regions in a layer-wise manner, and returns provably non-empty maximal CPA regions together with interior representatives. The framework further provides visualizations of region partitioning and decision boundaries, enabling qualitative inspection alongside quantitative region counts. By exploiting the affine restriction property of CPA networks under fixed activation patterns, AffineLens supports a broad class of modern components, including batch normalization, pooling, residual connections, multilayer perceptrons, and convolutional layers. Finally, we use AffineLens to perform a systematic empirical study of architectural expressivity, comparing networks through region complexity metrics and revealing how design choices influence the geometry of learned functions.
\end{abstract}


\section{Introduction}
Neural networks are a family of parametric function approximators that map inputs to outputs by composing simple transformations across multiple layers. By stacking linear (affine) maps with nonlinear activation functions, deep networks can represent highly complex input--output relationships and have achieved strong empirical success across vision, language, and control \cite{book36}. Despite this success, a central challenge remains: understanding why a trained network behaves the way it does and how its architecture and parameters shape the resulting decision geometry. Piecewise affine neural networks (PANNs), i.e., deep networks composed of affine maps and continuous piecewise affine (CPA) nonlinearities, form an especially important and widely used subclass. This includes many practical architectures built from ReLU-like activations (and their variants), for which the network function is continuous and piecewise affine over the input domain \cite{book48}. PANNs are valuable to study because they admit a concrete geometric interpretation: the input space is partitioned into a finite collection of affine convex regions, within each of which the network reduces to a single affine function, while global nonlinearity arises only through the combinatorial arrangement of these regions \cite{book1}. This region-based view offers an appealing route to understanding expressivity, robustness, and decision geometry beyond purely statistical summaries (e.g., accuracy). In particular, the affine-region decomposition provides an explicit mechanistic object that links architecture and parameters to the shape, complexity, and locality of decision boundaries. As a result, analyzing and enumerating affine regions can yield principled insights into how depth and width control function complexity, how sensitive predictions are to perturbations, and why certain decision surfaces emerge during training.

However, despite the central role of affine regions in PANN theory, precise and general tools for quantifying and inspecting these regions remain limited. Existing studies often rely on (i) theoretical bounds that can be loose and hard to connect to the behavior of a trained model \cite{book11,book12}, or (ii) sampling-based approximations that are intrinsically incomplete and may systematically miss small yet decision-critical regions, especially those concentrated near decision boundaries \cite{book9}. These limitations become more pronounced when moving from toy settings to practical architectures: batch normalization, pooling, residual connections, and convolutional components substantially complicate the induced hyperplane arrangement, and many prior methods are either restricted to specific network families or limited to enumeration in 2D settings \cite{book10}. Consequently, it is difficult to perform reproducible and fair comparisons of expressivity across architectures or to diagnose how and where a trained model forms intricate decision geometry.

Motivated by this gap, we seek a unified framework that can exactly enumerate the affine regions expressed by a given trained PANN on a task-relevant bounded domain. 
Concretely, we focus on a calibrated input polytope $A_0$, which captures the operational/data-supported region of interest and enables meaningful, comparable measurements across models. On such a domain, the network induces a well-defined hyperplane arrangement whose cells correspond to affine regions; therefore, exhaustive enumeration becomes a form of geometric auditing: it makes region complexity and decision-boundary structure observable quantities rather than indirect proxies. This capability is essential for answering questions such as how depth/width trade off in realized expressivity, how residual and convolutional inductive biases reshape partitions, and how region geometry evolves during training.

Building on the above geometric view and our goal of exact region auditing on a calibrated domain, we consider a practically relevant class of CPA activations that includes ReLU and its variants such as LeakyReLU \cite{book3}. Specifically, throughout this paper we use the two-slope element-wise activation
\begin{equation}
\label{eq:4}
\sigma(x)=
\begin{cases}
ax, & x>0,\\
bx, & x\le 0,
\end{cases}
\qquad a,b\in\mathbb{R}.
\end{equation}
This parameterization preserves the PANN structure while allowing a nonzero (and potentially learnable) negative slope, thereby enlarging the set of region-wise affine maps realized by the network \cite{book1,book4}. When batch normalization is present, we treat it in inference mode (with frozen statistics), in which case it reduces to an affine map and thus preserves the CPA structure \cite{book5}. Consequently, given a fixed input, the entire network can be mapped into a concrete geometric structure, a process fundamentally representing the hyperplane arrangement \cite{book6,book7} induced by the CPA components. We introduce a method for accurately calculating and enumerating the number of affine regions represented by PANNs. Our approach accommodates residual connections \cite{book19} and fully connected layers, extends to networks of arbitrary depth and width, and enables the enumeration of affine regions generated in high-dimensional spaces. Additionally, this method facilitates the localized representation of any calibrated region within the input space, offering insights into the current decision characteristics of the model. Thus, our method offers flexible and comprehensive analysis of the expressive capacity of various PANN structures, distinguishing itself from existing methods \cite{book9,book10}. Our contributions are summarized as follows:
\begin{itemize}
\item We propose AffineLens, a unified and general framework for analyzing PANNs by exactly computing and visualizing their affine region decomposition. AffineLens applies to high-dimensional input spaces and a wide variety of architectures.
\item We formulate the search for sub-affine regions as a linear programming (LP) problem and propose a region exploration algorithm grounded in breadth-first traversal over hyperplane arrangements. This method ensures exhaustive coverage of all affine subregions within a bounded input space, without relying on heuristic sampling strategies.
\item We conduct a quantitative empirical analysis of how neural networks, under varying architectures and training strategies, influence the representation of affine regions, thereby showcasing the flexibility of AffineLens.
\end{itemize}

\section{Related Work}

\subsection{Counting Affine Regions}
Understanding the geometric complexity of PANNs has attracted significant attention. Early efforts include \cite{book8}, which proposed an algorithm to compute the number of affine regions in maxout networks. \cite{book9} introduced a polyhedral extraction method that operates on subdivided edges rather than explicitly enumerating full regions. \cite{book10} proposed SplineCam, a technique grounded in PANN theory for high-fidelity computation of network geometry. In the context of graph neural networks, \cite{book11} derived tight upper and lower bounds on the number of affine regions. For convolutional ReLU networks, \cite{book12} analyzed both the maximal and average number of regions. Leveraging tropical geometry, \cite{book13} explicitly calculated the number of boundary and affine segments. To mitigate the exponential growth of affine regions with depth, \cite{book14} proposed an accuracy-based approach that reduces this growth to a polynomial function of width.

\subsection{Evolution of Affine Regions}
Beyond counting, the dynamics of affine-region formation have also been studied. \cite{book15,book55} employed a geometric framework to analyze how PANNs hierarchically organize input signals. \cite{book16} conducted empirical analyses of how the count and density of affine regions evolve in reinforcement learning tasks with continuous control. \cite{book17} examined the relationship between ReLU network architecture and the configuration of decision regions. In addition, \cite{book18} showed that uniform sampling becomes exponentially inefficient for recovering small-volume regions as input dimensionality grows.

\subsection{Expressive Capacity of Neural Networks}
The expressive capacity of PANNs has been investigated from multiple perspectives. The roles of network depth and width have been characterized in \cite{book20,book21,book22,book23,book43,book52}, while the influence of training parameters on expressiveness has been examined in \cite{book24,book25,book26,book27,book49,book53,book54}. Furthermore, extensions of the universal approximation theorem \cite{book35} have been developed in \cite{book30,book31,book32,book33,book51}, providing theoretical foundations for the representational power of PANNs.
Complementary to these global expressivity results, recent work on local complexity in ReLU networks characterizes expressive capacity through the density of linear regions over the input distribution and links it to representation learning and optimization dynamics \cite{book50}.

\subsection{Positioning of Our Method}
Compared with existing methods, our work differs in both methodology and scope. While prior studies often rely on sampling-based approximations or theoretical estimates \cite{book9,book11,book12}, we formulate affine-region enumeration as an LP-based procedure, enabling exact and exhaustive computation over a task-calibrated bounded polytope and returning region-wise geometric objects for downstream inspection. Moreover, unlike SplineCam \cite{book10}, which targets 2D slices for visualization, AffineLens supports enumeration in the original high-dimensional input space. Finally, in contrast to the arrangement-driven approach \cite{book18} that primarily enumerate feasible activation patterns, our framework is explicitly designed for region-level geometric auditing on a calibrated domain, facilitating localized analysis of decision geometry.

\section{Methodology}
This section introduces our proposed method, AffineLens, for computing, identifying, and visualizing affine regions in PANNs. We describe how AffineLens computes CPA function representations, followed by a formal description of our region enumeration algorithm.

\subsection{Notation}
\label{sec:notation}
We summarize the common notation used in the remainder of this section.

\paragraph{Network and dimensions}
The input dimension is $d_0$, the network depth is $L$, and the width of layer $l$ is $d_l$ for $l=1,\dots,L$.

\paragraph{Regions and counting}
$A_0\subset\mathbb{R}^{d_0}$ is the calibrated bounded input polytope (Eq.~\eqref{(4)}).
For a reference point $\hat{x}\in A_0$, $A_l^{\hat{x}}\subseteq\mathbb{R}^{d_0}$ denotes the layer-$l$ affine region induced by the sign pattern up to layer $l$ (Eq.~\eqref{(10)}).
We maintain a representative set $\widehat{X}_l$ (Eq.~\eqref{(11)}), and denote the number of distinct layer-$l$ regions by
$R_l \triangleq |\widehat{X}_l|$ (with $R_0=1$).

\paragraph{Hyperplanes}
Given a parent region $A_{l-1}^{\hat{x}}$, layer $l$ induces neuron hyperplanes
$\mathcal{H}_l^{\hat{x}}=\{h_{l,i}^{\hat{x}}\}_{i=1}^{d_l}$ (Eq.~\eqref{(12)}).
We write
\[
\widehat{\mathcal{H}}_l^{\hat{x}}
\triangleq \{h\in\mathcal{H}_l^{\hat{x}} \mid h\cap A_{l-1}^{\hat{x}}\neq \emptyset\},
\qquad
m_l \triangleq \big|\widehat{\mathcal{H}}_l^{\hat{x}}\big| \le d_l,
\]
and use $m_l=d_l$ for worst-case bounds.

\paragraph{Polytope constraints}
We represent a region/polytope as an intersection of $K$ linear inequalities, and denote by $K_0$ the number of inequalities defining $A_0$.
Let $K_l$ be the number of inequalities used to describe a typical layer-$l$ region in input space.

\paragraph{LP complexity model}
A linear program with $n$ decision variables and $K$ constraints costs time $T_{\mathrm{LP}}(n,K)$.
In particular, feasibility and facet-crossing checks typically use $n=d_0$, while a Chebyshev-center LP uses $n=d_0+1$ (radius variable included).

\begin{figure*}[t]
  \centering
  \subfloat[]{\safeincludegraphics[width=0.22\textwidth,height=0.20\textheight,keepaspectratio]{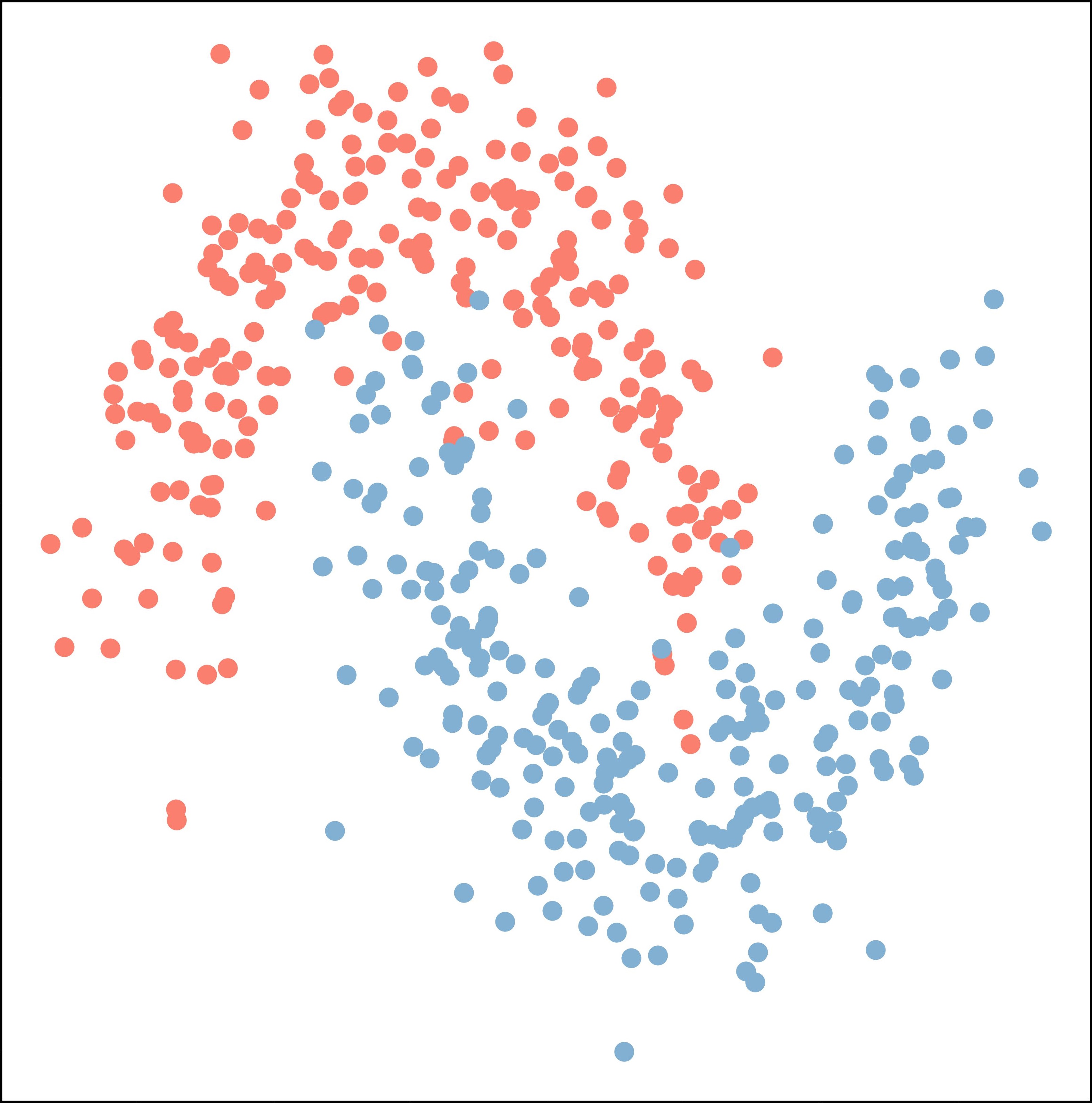}\label{fig:f1a}}\hfill
  \subfloat[]{\safeincludegraphics[width=0.22\textwidth,height=0.20\textheight,keepaspectratio]{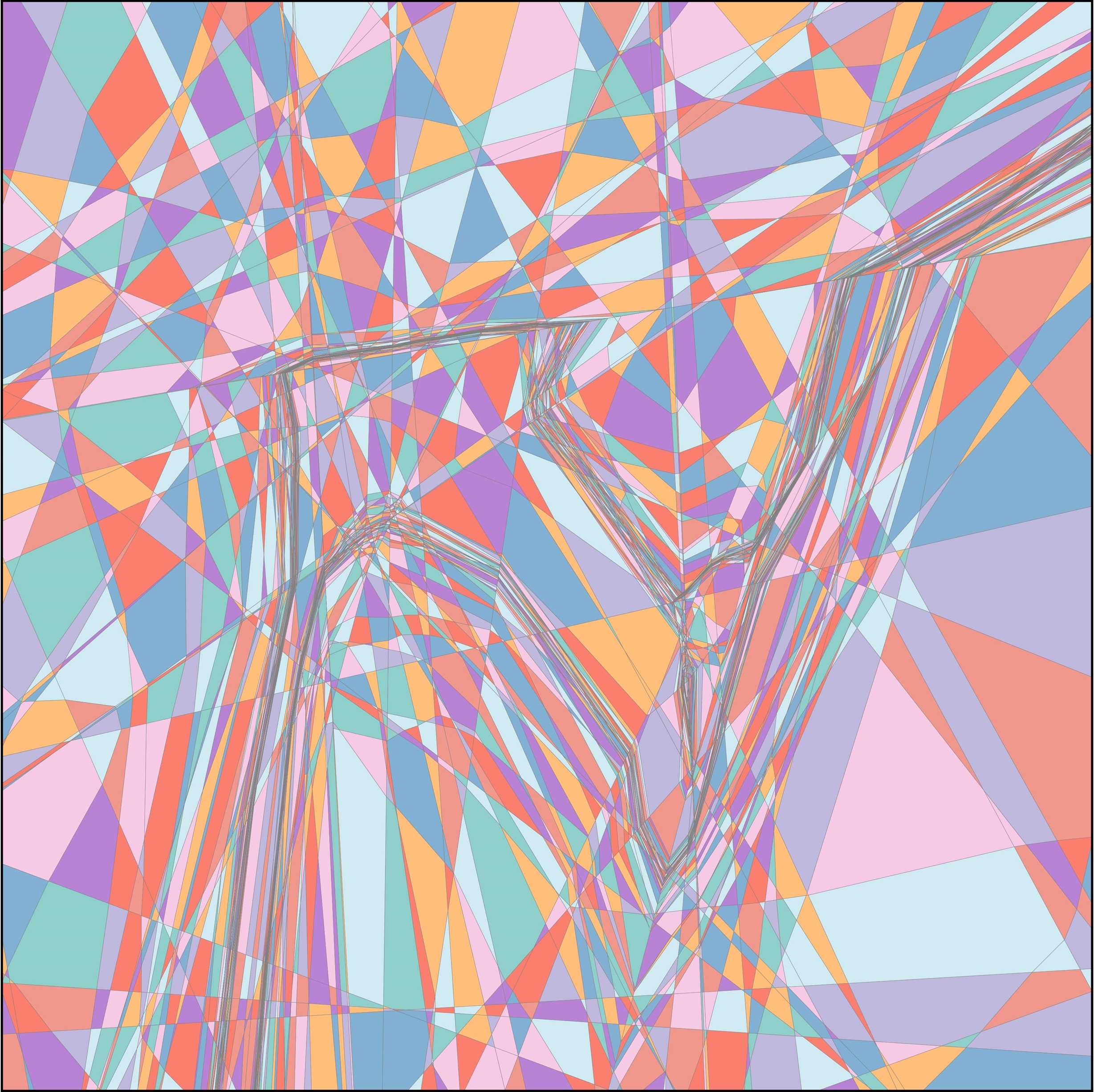}\label{fig:f1b}}\hfill
  \subfloat[]{\safeincludegraphics[width=0.22\textwidth,height=0.20\textheight,keepaspectratio]{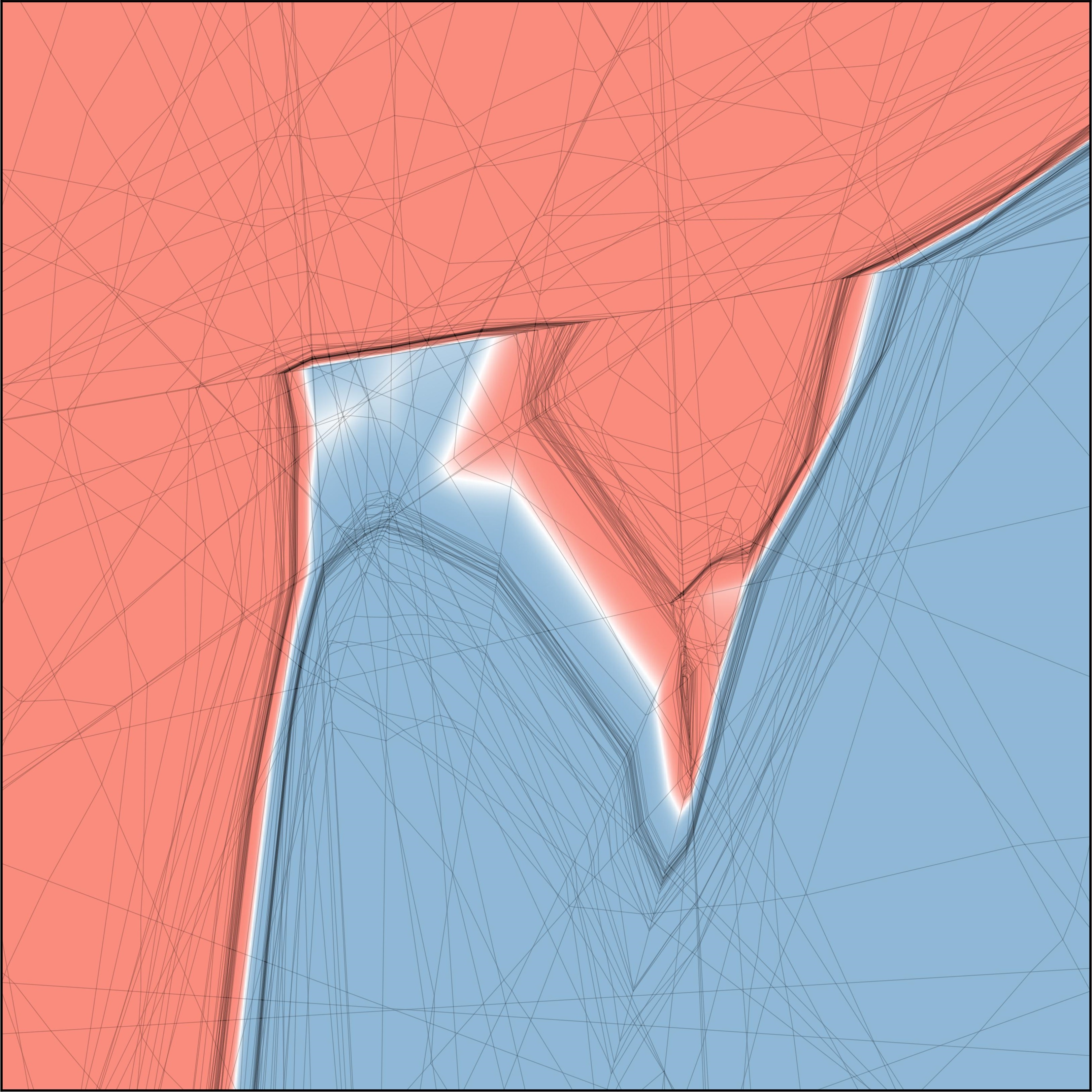}\label{fig:f1c}}\hfill
  \subfloat[]{\safeincludegraphics[width=0.22\textwidth,height=0.20\textheight,keepaspectratio]{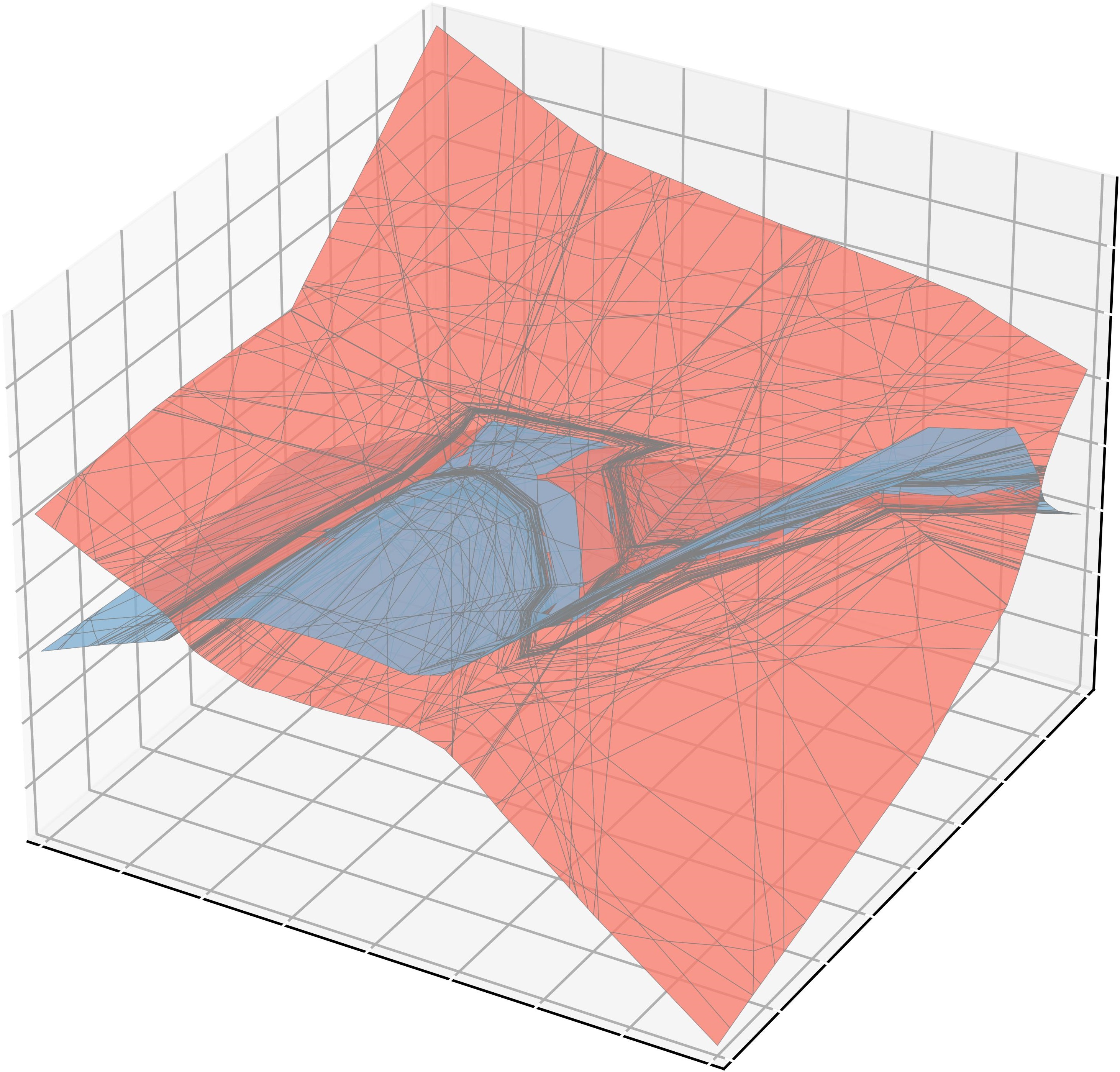}\label{fig:f1d}}

  \caption{Precise visualization of input distribution, affine regions, and decision boundaries in a PANN, shown in correspondence. (a) Input distribution of the network. (b) Geometric structure of the network in a 2-dimensional space, where each neuron induces a hyperplane and each distinct color represents an affine region. (c) The white band indicates the vicinity of the decision boundary, which is formed by intersections of multiple hyperplanes. The remaining 2 colors represent 2 categories of inputs. (d) To aid understanding of the decision boundary, a 3-dimensional visualization of affine regions and the decision boundary is presented.}
  \label{fig:f1}
\end{figure*}

\begin{figure}[t]
  \centering
  \subfloat[]{%
    \safeincludegraphics[width=0.235\linewidth,keepaspectratio]{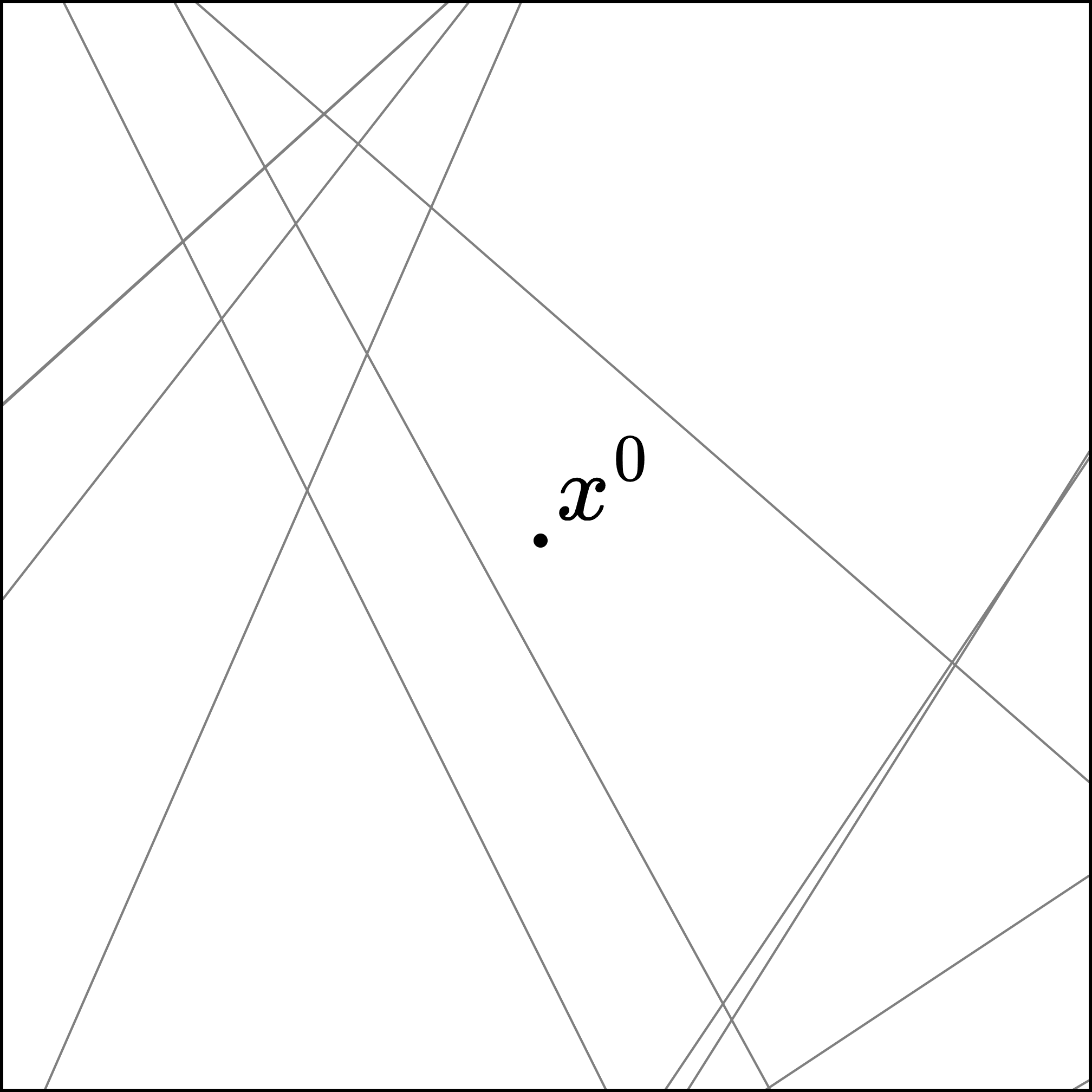}%
    \label{fig:f2a}%
  }\hfill
  \subfloat[]{%
    \safeincludegraphics[width=0.235\linewidth,keepaspectratio]{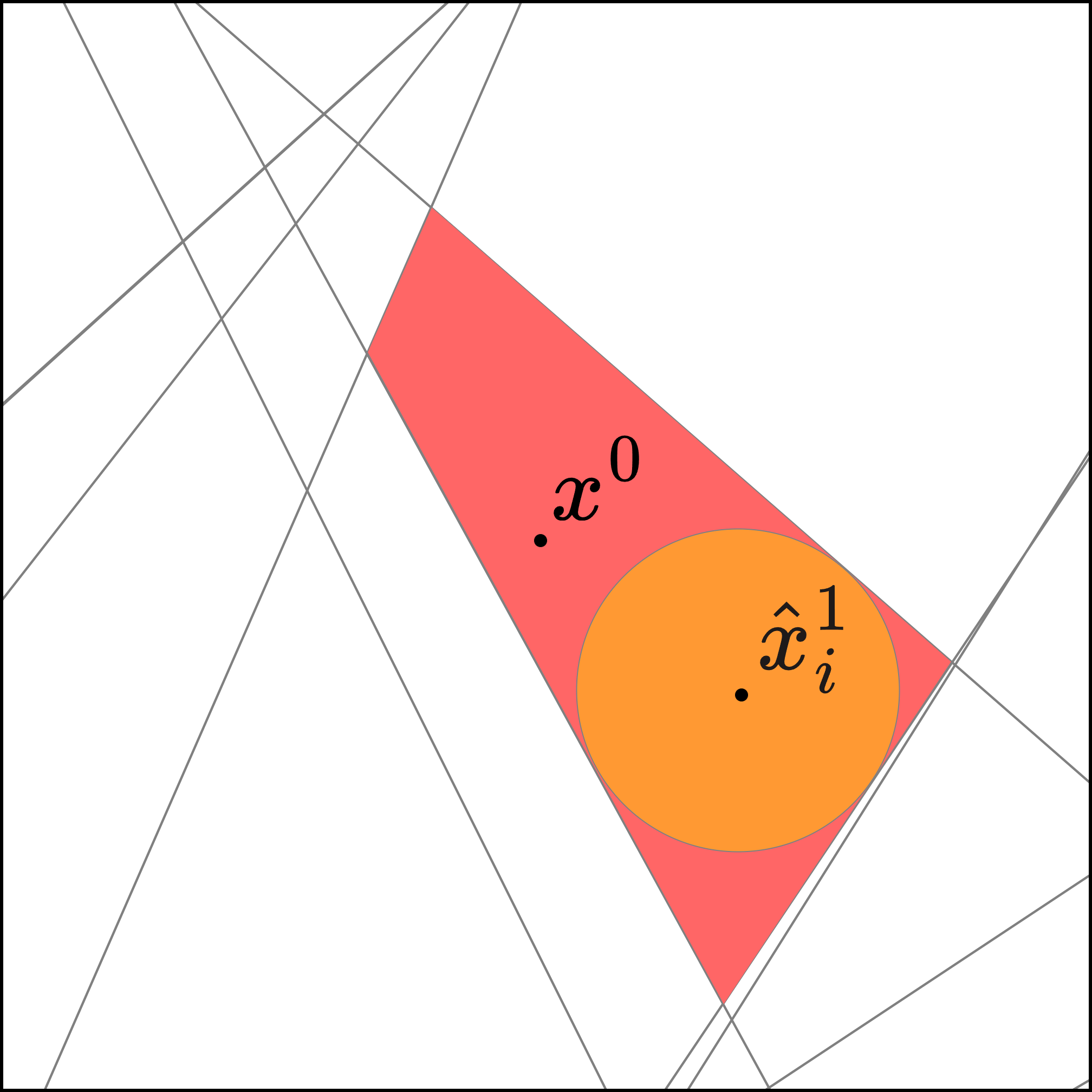}%
    \label{fig:f2b}%
  }\hfill
  \subfloat[]{%
    \safeincludegraphics[width=0.235\linewidth,keepaspectratio]{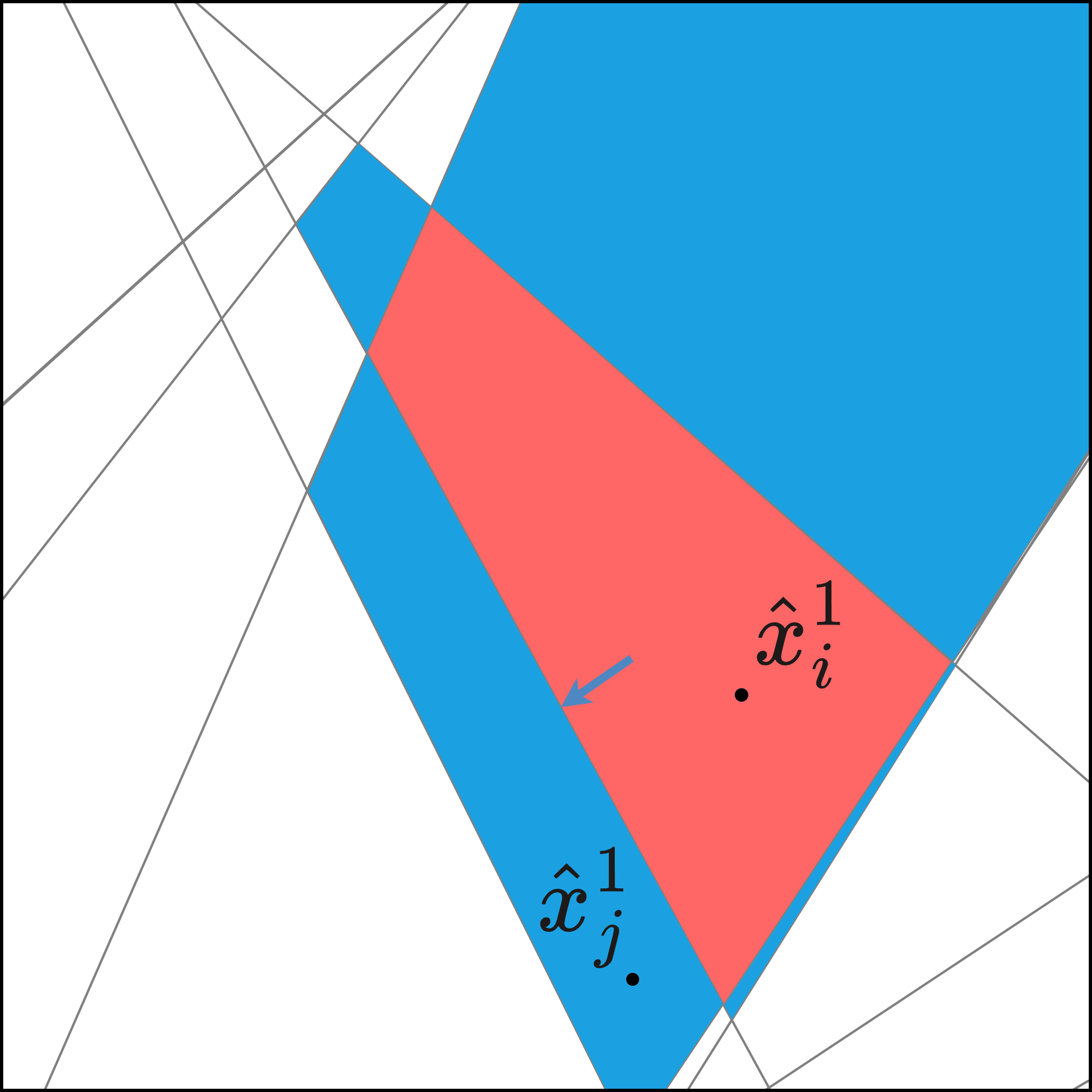}%
    \label{fig:f2c}%
  }\hfill
  \subfloat[]{%
    \safeincludegraphics[width=0.235\linewidth,keepaspectratio]{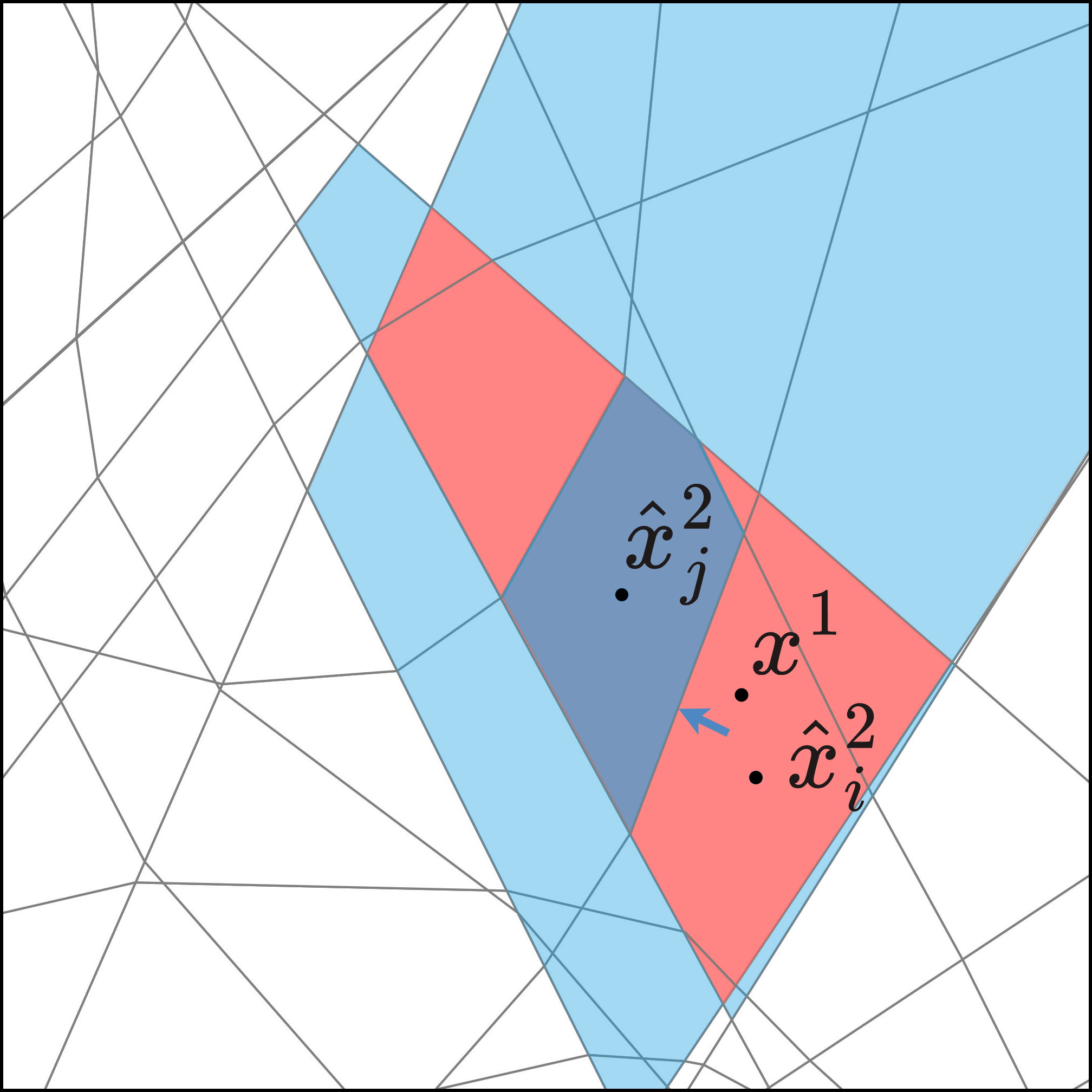}%
    \label{fig:f2d}%
  }

  \caption{Visualization of the affine region search. (a) Define an input space $A_{0}$ and a point $x^0$ within it; calculate the set of hyperplanes intersecting $A_{0}$ through a PANN. (b) Derive the sub-region $A^{\hat{x}_i}_{1}$ containing $x^0$ (highlighted in red) where ${\hat{x}_i}^{1}$ is the center of the largest inscribed ball (orange circle) of this region. Then, label it. (c) Explore the edges of the region $A^{\hat{x}_i}_{1}$ (blue arrow) and identify an unmarked adjacent region $A^{\hat{x}_j}_{1}$ (one of the blue regions). Repeat this process to identify all sub-regions of $A_{0}$. (d) Select a point $x^1$ within region $A^{\hat{x}_i}_{1}$ to explore, identifying its sub-region $A^{\hat{x}_i}_{2}$. Following the blue arrow, we then explore the neighboring region $A^{\hat{x}_j}_{2}$. In this way, we can delve into deeper layers and repeat the search.}
  \label{fig:f2}
\end{figure}

\subsection{Piecewise Affine Neural Networks}
A typical feedforward PANN can be viewed as a composition of affine maps and element-wise CPA activations.
Let $L\in\mathbb{N}$ denote the number of layers, with widths $\{d_l\}_{l=0}^{L}$, where $d_0$ and $d_L$ are the input and output dimensions, respectively.
For each layer $l\in\{1,\dots,L\}$, define the affine map $T_l:\mathbb{R}^{d_{l-1}}\rightarrow\mathbb{R}^{d_l}$ by
\begin{equation}
\label{(eq:1)}
T_l(x)=w_l x + b_l,
\end{equation}
where $w_l\in\mathbb{R}^{d_l\times d_{l-1}}$ and $b_l\in\mathbb{R}^{d_l}$.
Let $\sigma:\mathbb{R}^{d_l}\rightarrow\mathbb{R}^{d_l}$ denote an element-wise activation.
Then the network $NN:\mathbb{R}^{d_0}\rightarrow\mathbb{R}^{d_L}$ is
\begin{equation}
\label{eq:2}
NN(x)=T_L\circ \sigma \circ T_{L-1}\circ \sigma \circ \cdots \circ \sigma \circ T_1(x).
\end{equation}
Equivalently, letting $x_0=x$, we have
\begin{equation}
\label{eq:3}
\begin{aligned}
x_l &= \sigma\!\left(T_l(x_{l-1})\right), \quad l=1,\dots,L-1,\\
NN(x) &= T_L(x_{L-1}).
\end{aligned}
\end{equation}

The affine regions of $NN:\mathbb{R}^{d_0}\rightarrow\mathbb{R}^{d_L}$ are defined as the maximal connected subsets of the input space on which $NN$ is affine.
Generically, affine regions are full-dimensional $d_0$; throughout we restrict attention to full-dimensional regions and ignore measure-zero boundaries and degeneracies \cite{book1}.
Fig.~\ref{fig:f1} illustrates the affine-region partition and decision geometry of a ReLU \cite{book37} network in a 2-dimensional input space for binary classification.
Each colored block represents a distinct affine region, and the decision boundary is formed by intersections of multiple hyperplanes induced across layers.
As depth increases, neurons in each successive layer further subdivide regions formed by previous layers, creating a hierarchical refinement of the input partition (Fig.~\ref{fig:f2}).
These densely distributed affine regions provide a convenient and geometry-faithful lens for analyzing network expressivity, decision complexity, and their dependence on architecture and training.

\subsection{The AffineLens Framework}
AffineLens begins from a calibrated input space, detects active hyperplanes layer-by-layer, and enumerates affine
regions via breadth-first search. To facilitate the introduction of AffineLens, we present the relevant definitions
and notation for CPA functions and PANNs.

\paragraph{Calibrated input space}
We define the calibrated bounded input space as a closed polytope
\begin{equation}
\label{(4)}
A_0=\left\{x\in\mathbb{R}^{d_0}\ \middle|\ {\bf W}_0x+{\bf b}_0\ge 0\right\},
\end{equation}
where ${\bf W}_0 \in \mathbb{R}^{K_0 \times d_0}$, ${\bf b}_0 \in \mathbb{R}^{K_0}$, and $K_0 \in \mathbb{N}_{+}$.
The inequality ${\bf W}_0x+{\bf b}_0\ge 0$ is interpreted element-wise.
We assume $A_0$ is bounded; in practice this can be ensured by including explicit calibration constraints such as box bounds (e.g., $x_{\min}\le x\le x_{\max}$) in the H-representation.
Using non-strict inequalities makes $A_0$ a standard H-representation polytope, which is directly compatible with LP-based feasibility tests and facet-crossing checks used throughout our enumeration algorithm.
When discussing region partitions, all statements are understood up to measure-zero boundaries, and our representative points are chosen from the relative interior of each region.

\paragraph{Layer-wise affine representation}
Fix a reference input $\hat{x}\in A_0$. On the region induced by the sign pattern up to layer $l\!-\!1$, the
pre-activation vector of the $l$-th layer is an affine function of the input:
\begin{equation}
\label{(5)}
f^{\hat{x}}_l(x)= {\bf W}^{\hat{x}}_l x + {\bf b}^{\hat{x}}_l,\quad x\in A_{l-1}^{\hat{x}},
\end{equation}
where $A_{l-1}^{\hat{x}}\subseteq \mathbb{R}^{d_0}$ denotes the affine region determined by the
activation/sign pattern up to layer $l\!-\!1$ (and we set $A_0^{\hat{x}}:=A_0$).

\paragraph{CPA activation and slope matrices}
Let $\operatorname{sgn}(\cdot)$ be the sign function
\begin{equation}
\operatorname{sgn}(t)=
\begin{cases}
1, & t>0,\\
-1, & t\le 0,
\end{cases}
\end{equation}
and define the (element-wise) two-slope selector
\begin{equation}
\label{(8)}
\alpha(t)=\frac{a(\operatorname{sgn}(t)+1)-b(\operatorname{sgn}(t)-1)}{2},
\end{equation}
so that $\alpha(t)=a$ for $t>0$ and $\alpha(t)=b$ for $t\le 0$. Accordingly, the CPA activation can be written
as $\sigma(t)=\alpha(t)\,t$.

For a fixed reference input $\hat{x}$, the diagonal slope matrix of the activation following layer $l$ is defined by
\begin{equation}
\label{(9)}
\boldsymbol{\Gamma}_l^{\hat{x}}
=\operatorname{diag}\!\big(\alpha(f_l^{\hat{x}}(\hat{x}))\big),\qquad l=1,\dots,L-1,
\end{equation}
where $\alpha(\cdot)$ is applied element-wise and $\operatorname{diag}(\cdot)$ maps a vector to a diagonal matrix.

\paragraph{Closed-form effective parameters}
Let $\{w_l,b_l\}_{l=1}^{L}$ denote the network weights and biases. Fix a reference input $\hat{x}\in A_0$ and
consider the region where the activation pattern up to layer $l\!-\!1$ is fixed. On this region, the
pre-activation at layer $l$ admits the affine form in Eq.~\eqref{(5)},
$f_l^{\hat{x}}(x)={\bf W}_l^{\hat{x}}x+{\bf b}_l^{\hat{x}}$.

\begin{itemize}
  \item \emph{(Slope matrices)} For each $t=1,\dots,L-1$, the diagonal slope matrix is
  \begin{equation}
  \label{eq:D-def}
  \boldsymbol{\Gamma}_t^{\hat{x}}
  \;\triangleq\;
  \operatorname{diag}\!\big(\alpha(f_t^{\hat{x}}(\hat{x}))\big).
  \end{equation}

\item \emph{(Effective weights)} For $l=1$, ${\bf W}_1^{\hat{x}}=w_1$. For $l\ge 2$,
\begin{equation}
\label{eq:W-eff-chain}
\begin{aligned}
{\bf W}_l^{\hat{x}}
\;=\;
w_l\,\boldsymbol{\Gamma}_{l-1}^{\hat{x}}\,w_{l-1}\,\boldsymbol{\Gamma}_{l-2}^{\hat{x}}\cdots w_2\,\boldsymbol{\Gamma}_1^{\hat{x}}\,w_1
\\
\;=\;
\Big(\prod_{i=2}^{l} w_i\,\boldsymbol{\Gamma}_{i-1}^{\hat{x}}\Big)\,w_1.
\end{aligned}
\end{equation}

  \item \emph{(Effective biases)} For $l=1$, ${\bf b}_1^{\hat{x}}=b_1$. For $l\ge 2$, define the partial effective
  map
  \begin{equation}
  \label{eq:W-partial}
  {\bf W}_k^{(l,\hat{x})}
  \;\triangleq\;
  w_l\,\boldsymbol{\Gamma}_{l-1}^{\hat{x}}\,w_{l-1}\,\boldsymbol{\Gamma}_{l-2}^{\hat{x}}\cdots w_k,
  \qquad k=2,\dots,l,
  \end{equation}
  so that ${\bf W}_l^{(l,\hat{x})}=w_l$ and ${\bf W}_{l-1}^{(l,\hat{x})}=w_l\boldsymbol{\Gamma}_{l-1}^{\hat{x}}w_{l-1}$.
  Then the effective bias admits the equivalent expansions
  \begin{equation}
  \label{eq:b-eff-expanded}
  \begin{aligned}
  {\bf b}_l^{\hat{x}}
  \;=\;
  b_l
  &+{\bf W}_l^{(l,\hat{x})}\,\boldsymbol{\Gamma}_{l-1}^{\hat{x}}\,b_{l-1}\\
  &+{\bf W}_{l-1}^{(l,\hat{x})}\,\boldsymbol{\Gamma}_{l-2}^{\hat{x}}\,b_{l-2}\\
  &+\ \cdots\\
  &+{\bf W}_{2}^{(l,\hat{x})}\,\boldsymbol{\Gamma}_{1}^{\hat{x}}\,b_{1},
  \end{aligned}
  \end{equation}
  and equivalently,
  \begin{equation}
  \label{eq:b-eff-sum}
  \begin{aligned}
  {\bf b}_l^{\hat{x}}
  \;=\;
  b_l
  +\sum_{i=1}^{l-1}
  \Big(\prod_{j=i+1}^{l} w_j\,\boldsymbol{\Gamma}_{j-1}^{\hat{x}}\Big)\,b_i.
  \end{aligned}
  \end{equation}
\end{itemize}

\begin{algorithm}[tb]
\caption{Find CPAs}
\label{alg:find-cpas}
\textbf{Input}: $NN=(L,\{w_l,b_l\}_{l=1}^L)$ a PANN; an initial point $x\in A_0$\\
\textbf{Output}: $E_c$, a set of pairs $(\hat{x},A_L^{\hat{x}})$ representing all (non-empty) CPAs on $A_0$
\begin{algorithmic}[1]
\STATE Initialize $E_c\gets\emptyset$, $Q\gets\emptyset$ \COMMENT{$Q$ is a queue (or stack) of regions to explore}
\STATE PUSH $Q\leftarrow (0, x, A_0)$
\WHILE{$Q\neq\emptyset$}
    \STATE POP $(l, x_i, A_l^{\hat{x}_i}) \leftarrow Q$
    \IF{$l=L$}
        \STATE APPEND $E_c \leftarrow (\hat{x}_i, A_L^{\hat{x}_i})$
    \ELSE
        \STATE $E_s \gets \texttt{SearchSubRegions}(NN,\,l,\,A_l^{\hat{x}_i},\,x_i)$ \COMMENT{Alg.~\ref{alg:search-sub-regions}}
        \COMMENT{$E_s$ contains pairs $(\hat{x}_j, A_{l+1}^{\hat{x}_j})$}
        \FORALL{$(\hat{x}_j, A_{l+1}^{\hat{x}_j}) \in E_s$}
            \STATE PUSH $Q \leftarrow (l+1, \hat{x}_j, A_{l+1}^{\hat{x}_j})$
        \ENDFOR
    \ENDIF
\ENDWHILE
\end{algorithmic}
\end{algorithm}

\begin{algorithm}[tb]
\caption{Search sub-regions}
\label{alg:search-sub-regions}
\textbf{Input}: a PANN $NN=(L,\{w_t,b_t\}_{t=1}^L)$; layer index $l\in\{0,\dots,L-1\}$; a parent region $A_l^{\hat{x}}\subseteq\mathbb{R}^{d_0}$; and a point $x\in A_l^{\hat{x}}$\\
\textbf{Output}: $E_s$, a set of pairs $(\hat{x}_j, A^{\hat{x}_j}_{l+1})$ enumerating all non-empty sub-regions of the parent region $A_l^{\hat{x}}$
\begin{algorithmic}[1]
\STATE Initialize $E_s\gets\emptyset$ \COMMENT{output set}
\STATE Initialize $Q\gets\emptyset$ \COMMENT{queue (or stack) of regions to explore}
\STATE Initialize $V\gets\emptyset$ \COMMENT{visited set of regions (canonical constraint systems)}
\vspace{0.25em}

\STATE Compute the effective parameters $( {\bf W}_{l+1}^{\hat{x}}, {\bf b}_{l+1}^{\hat{x}} )$ using Eqs.~\eqref{eq:W-eff-chain} and \eqref{eq:b-eff-sum}
\STATE Define the neuron hyperplanes on the parent region:
$\mathcal{H} \gets \{h_i\}_{i=1}^{d_{l+1}}$ with $h_i=\{z\mid ({\bf W}_{l+1}^{\hat{x}})_i z + ({\bf b}_{l+1}^{\hat{x}})_i = 0\}$
\STATE Compute the intersecting subset
$\widehat{\mathcal{H}}\gets\{h_i\in\mathcal{H}\mid h_i\cap A_l^{\hat{x}}\neq\emptyset\}$
via LP feasibility (or min/max) checks
\vspace{0.25em}

\IF{$\widehat{\mathcal{H}}=\emptyset$}
    \STATE $A_{l+1}^{\hat{x}} \gets A_l^{\hat{x}}$
    \STATE \textbf{return} $E_s=\{(\hat{x},A_{l+1}^{\hat{x}})\}$ \COMMENT{no split occurs}
\ENDIF
\vspace{0.25em}

\STATE Compute the child region containing $x$ using Eq.~(\ref{(10)}):
$A_{l+1}^{x}\gets \texttt{RegionFromSign}(x,\,A_l^{\hat{x}},\,{\bf W}_{l+1}^{\hat{x}},\,{\bf b}_{l+1}^{\hat{x}})$
\STATE PUSH $Q\leftarrow A_{l+1}^{x}$
\vspace{0.25em}

\WHILE{$Q\neq\emptyset$}
    \STATE POP $A_{l+1} \leftarrow Q$
    \STATE Compute an interior representative $\hat{x}_i \gets \texttt{chebyshev\_center}(A_{l+1})$
    \COMMENT{$\hat{x}_i\in A_{l+1}$ (interior if feasible with $\epsilon>0$)}
    \STATE Canonicalize constraints: $\widetilde{A}_{l+1} \gets \texttt{remove\_redundancy}(A_{l+1})$
    \COMMENT{$\widetilde{A}_{l+1}=A_{l+1}$ as sets, with fewer inequalities}
    \IF{$\widetilde{A}_{l+1}\in V$}
        \STATE \textbf{continue}
    \ENDIF
    \STATE INSERT $V \leftarrow \widetilde{A}_{l+1}$
    \STATE APPEND $E_s \leftarrow (\hat{x}_i,\,\widetilde{A}_{l+1})$
    \vspace{0.25em}

    \STATE Compute all feasible, previously unseen neighbor regions across facets induced by $\widehat{\mathcal{H}}$:
    \STATE $\mathcal{N}\gets \texttt{Neighbors}(\widetilde{A}_{l+1},\,A_l^{\hat{x}},\,\widehat{\mathcal{H}})$
    \FORALL{$A' \in \mathcal{N}$}
        \STATE PUSH $Q \leftarrow A'$
    \ENDFOR
\ENDWHILE
\STATE \textbf{return} $E_s$
\end{algorithmic}
\end{algorithm}

\paragraph{Affine region induced by layer $l$}
Given a reference point $\hat{x}\in A_0$, the layer-$l$ region in input space is defined as the subset of the parent region $A_{l-1}^{\hat{x}}$ on which the pre-activation sign pattern at layer $l$ matches that of $\hat{x}$:
\begin{equation}
\label{(10)}
A_{l}^{\hat{x}}
\;\triangleq\;
\left\{x\in A_{l-1}^{\hat{x}}\ \middle|\ 
\operatorname{diag}\!\big(\operatorname{sgn}(f^{\hat{x}}_l(\hat{x}))\big)\,f^{\hat{x}}_l(x)\ge 0
\right\}.
\end{equation}
Equivalently, Eq.~\eqref{(10)} enforces $\operatorname{sgn}(f^{\hat{x}}_l(x))=\operatorname{sgn}(f^{\hat{x}}_l(\hat{x}))$ element-wise for all $x\in A_{l}^{\hat{x}}$, and thus fixes the activation pattern up to layer $l$.
In our enumeration, we count \emph{full-dimensional} regions and treat the boundaries (where equalities hold) as measure-zero; representative points are always chosen from the relative interior, so the induced sign pattern is unambiguous.

\begin{theorem}[Affine realization under a fixed activation pattern]
\label{thm:affine-realization}
Fix $\hat{x}\in A_0$ and a layer index $l\in\{1,\dots,L\}$.
Let the activation slopes be piecewise-constant as in Eq.~\eqref{(8)}, and define the diagonal slope matrices
$\boldsymbol{\Gamma}_t^{\hat{x}}$ by Eq.~\eqref{eq:D-def}.
Then the layer-$l$ pre-activation admits the affine representation in Eq.~\eqref{(5)} on $A_{l-1}^{\hat{x}}$,
with effective parameters ${\bf W}_l^{\hat{x}}$ and ${\bf b}_l^{\hat{x}}$ given by
Eqs.~\eqref{eq:W-eff-chain} and \eqref{eq:b-eff-sum}.
Moreover, the induced set $A_l^{\hat{x}}$ in Eq.~\eqref{(10)} is a closed polytope in input space.
On every maximal region $A_L^{\hat{x}}$, the full network mapping $NN$ is affine in $x$.
\end{theorem}

\paragraph{Representative set and region counting}
We maintain a representative set of points $\widehat{X}_l\subseteq A_0$ such that each layer-$l$ affine region has exactly one representative:
\begin{equation}
\label{(11)}
\begin{aligned}
&\widehat{X}_l\subseteq A_0,\qquad
A_l^{x_i}\neq A_l^{x_j}\ \ (x_i\neq x_j,\ x_i,x_j\in\widehat{X}_l),\\
&\forall \hat{x}\in A_0,\ \exists x_i\in \widehat{X}_l\ \text{s.t.}\ A_l^{\hat{x}}=A_l^{x_i}.
\end{aligned}
\end{equation}
Therefore, the number of distinct affine regions at layer $l$ is $|\widehat{X}_l|$.

\paragraph{Hyperplane set at layer $l$}
Within $A_{l-1}^{\hat{x}}$, each neuron in layer $l$ induces a hyperplane where its pre-activation equals zero. Denote the set of such hyperplanes by
$\mathcal{H}_l^{\hat{x}}=\{h_{l,i}^{\hat{x}}\}_{i=1}^{d_l}$
where
\begin{equation}
\label{(12)}
h_{l,i}^{\hat{x}}=\left\{x\ \middle|\ {\bf W}^{\hat{x}}_{l,i}x+{\bf b}^{\hat{x}}_{l,i}=0\right\}.
\end{equation}
Here ${\bf W}^{\hat{x}}_{l,i}$ denotes the $i$-th row of ${\bf W}^{\hat{x}}_l$ and ${\bf b}^{\hat{x}}_{l,i}$ denotes the $i$-th component of ${\bf b}^{\hat{x}}_l$.

\begin{proof}
We first fix notation. For an input $x\in\mathbb{R}^{d_0}$, let $x_0(x)\triangleq x$, and define the
\emph{pre-activation} and \emph{post-activation} at layer $l$ by
\begin{equation}
z_l(x)\;\triangleq\; w_l\,x_{l-1}(x)+b_l,\qquad l=1,\dots,L,
\end{equation}
and, for $l=1,\dots,L-1$,
\begin{equation}
x_l(x)\;\triangleq\;\sigma\!\big(z_l(x)\big).
\end{equation}
(There is no activation after the last affine map $T_L$, so $NN(x)=z_L(x)$.) The activation is coordinate-wise
CPA with two slopes, i.e., $\sigma(t)=\alpha(t)\,t$ applied element-wise.

Fix a reference point $\hat{x}\in A_0$, and consider the region $A_{l-1}^{\hat{x}}$ induced by the sign pattern
up to layer $l-1$. Throughout, we take $\hat{x}$ (and all representatives) from the relative interior of its
region, so no relevant pre-activation coordinate is exactly zero on $\hat{x}$.
For each $t=1,\dots,l-1$, the sign pattern of $z_t(x)$ is constant over $A_{l-1}^{\hat{x}}$; hence the slope
selector $\alpha(\cdot)$ is constant over $A_{l-1}^{\hat{x}}$ coordinate-wise. Equivalently, the diagonal
matrices $\boldsymbol{\Gamma}_t^{\hat{x}}$ defined in \eqref{eq:D-def} are constant on $A_{l-1}^{\hat{x}}$, and
\begin{equation}
\begin{aligned}
x_t(x)
&=\sigma\!\big(z_t(x)\big)
=\boldsymbol{\Gamma}_t^{\hat{x}}\,z_t(x),
\qquad \forall x\in A_{l-1}^{\hat{x}},\\
&\hspace{7.6em} t=1,\dots,l-1.
\end{aligned}
\end{equation}

\medskip
\noindent\textbf{(1) Affine form of the layer-$l$ pre-activation and closed-form effective parameters.}
We prove by induction on $l$ that on $A_{l-1}^{\hat{x}}$ there exist matrices $\mathbf{W}_l^{\hat{x}}$ and
vectors $\mathbf{b}_l^{\hat{x}}$ such that
\begin{equation}\label{eq:zl-affine-proof}
z_l(x)= f_l^{\hat{x}}(x)=\mathbf{W}_l^{\hat{x}}x+\mathbf{b}_l^{\hat{x}},
\qquad \forall x\in A_{l-1}^{\hat{x}}.
\end{equation}

\emph{Base case ($l=1$).}
By definition, $z_1(x)=w_1x+b_1$, so we may set $\mathbf{W}_1^{\hat{x}}=w_1$ and $\mathbf{b}_1^{\hat{x}}=b_1$.

\emph{Inductive step.}
Assume \eqref{eq:zl-affine-proof} holds at layer $l-1\ge 1$, i.e., for all $x\in A_{l-2}^{\hat{x}}$,
\begin{equation}
z_{l-1}(x)=\mathbf{W}_{l-1}^{\hat{x}}x+\mathbf{b}_{l-1}^{\hat{x}}.
\end{equation}
Since $A_{l-1}^{\hat{x}}\subseteq A_{l-2}^{\hat{x}}$ and the sign pattern at layer $l-1$ is fixed over
$A_{l-1}^{\hat{x}}$, the slope matrix $\boldsymbol{\Gamma}_{l-1}^{\hat{x}}$ is constant on $A_{l-1}^{\hat{x}}$
and
\begin{equation}
x_{l-1}(x)=\sigma\!\big(z_{l-1}(x)\big)=\boldsymbol{\Gamma}_{l-1}^{\hat{x}}\,z_{l-1}(x),
\qquad \forall x\in A_{l-1}^{\hat{x}}.
\end{equation}
Therefore, for all $x\in A_{l-1}^{\hat{x}}$,
\begin{equation}
\begin{aligned}
z_l(x)
&= w_l\,x_{l-1}(x)+b_l\\
&= w_l\,\boldsymbol{\Gamma}_{l-1}^{\hat{x}}\,z_{l-1}(x)+b_l\\
&= w_l\,\boldsymbol{\Gamma}_{l-1}^{\hat{x}}\big(\mathbf{W}_{l-1}^{\hat{x}}x+\mathbf{b}_{l-1}^{\hat{x}}\big)+b_l\\
&= \big(w_l\,\boldsymbol{\Gamma}_{l-1}^{\hat{x}}\mathbf{W}_{l-1}^{\hat{x}}\big)x
 + \big(b_l+w_l\,\boldsymbol{\Gamma}_{l-1}^{\hat{x}}\mathbf{b}_{l-1}^{\hat{x}}\big).
\end{aligned}
\end{equation}
Hence we may define
\begin{equation}\label{eq:Wb-rec-proof}
\mathbf{W}_l^{\hat{x}} \triangleq w_l\,\boldsymbol{\Gamma}_{l-1}^{\hat{x}}\mathbf{W}_{l-1}^{\hat{x}},
\qquad
\mathbf{b}_l^{\hat{x}} \triangleq b_l+w_l\,\boldsymbol{\Gamma}_{l-1}^{\hat{x}}\mathbf{b}_{l-1}^{\hat{x}},
\end{equation}
which establishes \eqref{eq:zl-affine-proof} for layer $l$ and completes the induction.

Unrolling the recursion \eqref{eq:Wb-rec-proof} yields, for $l\ge 2$,
\begin{equation}
\mathbf{W}_l^{\hat{x}}
=
w_l\,\boldsymbol{\Gamma}_{l-1}^{\hat{x}}\,w_{l-1}\,\boldsymbol{\Gamma}_{l-2}^{\hat{x}}\cdots
w_2\,\boldsymbol{\Gamma}_1^{\hat{x}}\,w_1,
\end{equation}
which is \eqref{eq:W-eff-chain}, and repeated substitution of the bias recursion yields
\begin{equation}
\mathbf{b}_l^{\hat{x}}
=
b_l
+\sum_{i=1}^{l-1}
\Big(\prod_{j=i+1}^{l} w_j\,\boldsymbol{\Gamma}_{j-1}^{\hat{x}}\Big)\,b_i,
\end{equation}
which is \eqref{eq:b-eff-sum}.

\medskip
\noindent\textbf{(2) Polyhedrality (and boundedness) of $A_l^{\hat{x}}$.}
By definition \eqref{(10)},
\begin{equation}
A_l^{\hat{x}}
=
\left\{x\in A_{l-1}^{\hat{x}}\ \middle|\
\operatorname{diag}\!\big(\operatorname{sgn}(f_l^{\hat{x}}(\hat{x}))\big)\,f_l^{\hat{x}}(x)\ge 0
\right\}.
\end{equation}
On $A_{l-1}^{\hat{x}}$, Part (1) gives $f_l^{\hat{x}}(x)=\mathbf{W}_l^{\hat{x}}x+\mathbf{b}_l^{\hat{x}}$, hence
\begin{equation}
A_l^{\hat{x}}
=
A_{l-1}^{\hat{x}}
\cap
\bigcap_{i=1}^{d_l}
\left\{
x\ \middle|\
\operatorname{sgn}\!\big(f_{l,i}^{\hat{x}}(\hat{x})\big)\,
\big(\mathbf{W}_{l,i}^{\hat{x}}x+\mathbf{b}_{l,i}^{\hat{x}}\big)\ge 0
\right\}.
\end{equation}
Thus $A_l^{\hat{x}}$ is an intersection of finitely many closed halfspaces with $A_{l-1}^{\hat{x}}$, and by
induction from the H-representation \eqref{(4)} of $A_0$, it follows that $A_l^{\hat{x}}$ is a closed polyhedron
in input space. Moreover, $A_l^{\hat{x}}\subseteq A_0$ and $A_0$ is bounded by assumption, so $A_l^{\hat{x}}$ is
bounded; therefore $A_l^{\hat{x}}$ is a closed polytope.

\medskip
\noindent\textbf{(3) Affine realization of $NN$ on maximal regions.}
Let $A_L^{\hat{x}}$ denote any leaf (maximal) region produced by fixing the activation pattern through layer
$L-1$ (as in the main text, such regions are the CPAs returned by the enumeration procedure). On $A_L^{\hat{x}}$,
all slope matrices $\boldsymbol{\Gamma}_t^{\hat{x}}$ for $t=1,\dots,L-1$ are constant. Applying Part (1) with
$l=L$ gives that the last pre-activation $z_L(x)$ is affine in $x$ on $A_{L-1}^{\hat{x}}$, and hence also on its
subset $A_L^{\hat{x}}$. Since $NN(x)=z_L(x)$ (no activation after layer $L$), it follows that $NN$ is affine in
$x$ on $A_L^{\hat{x}}$.

Combining Parts (1)--(3) proves the Theorem~\ref{thm:affine-realization}.
\end{proof}

\begin{theorem}[Cell decomposition and BFS completeness]
\label{thm:bfs-completeness}
Let $P\subset\mathbb{R}^{d_0}$ be a non-empty \emph{full-dimensional} convex polytope (in our setting, $P=A_l^{\hat{x}}$),
and let $\widehat{\mathcal H}=\{h_i\}_{i=1}^{m}$ be the set of affine hyperplanes that intersect $P$.
Write each hyperplane as $h_i=\{x\mid g_i(x)=0\}$ where $g_i$ is affine.

For $\boldsymbol{\eta}\in\{\pm1\}^m$, define the (open) arrangement cell inside $P$ by
\begin{equation}
\mathcal{C}^\circ(\boldsymbol{\eta})
\;\triangleq\;
P \cap \bigcap_{i=1}^{m} \left\{ x \ \middle|\ \eta_i\,g_i(x) > 0\right\}.
\end{equation}
Then:
\begin{enumerate}
\item The non-empty \emph{full-dimensional} sets $\{\mathcal{C}^\circ(\boldsymbol{\eta})\}$ are exactly the connected components of
\begin{equation}
U \;\triangleq\; P\setminus \bigcup_{i=1}^{m}(h_i\cap P).
\end{equation}
and hence they partition $P$ up to a boundary set of Lebesgue measure zero.
\item Let $G$ be the adjacency graph whose nodes are the non-empty full-dimensional cells $\mathcal{C}^\circ(\boldsymbol{\eta})$,
and where two nodes are connected if the \emph{closures} of the corresponding cells share a common facet contained in $h_i\cap P$
for some $i$.
Then $G$ is connected. Consequently, starting from any interior point $x\in P$ and repeatedly crossing feasible
hyperplane-induced facets enumerates all non-empty full-dimensional cells in $P$.
\end{enumerate}
\end{theorem}

\begin{proof}
Let $P\subset\mathbb{R}^{d_0}$ be a non-empty full-dimensional convex polytope and
$\widehat{\mathcal{H}}=\{h_i\}_{i=1}^m$ be the set of affine hyperplanes intersecting $P$, where each
\begin{equation}
h_i=\{x\mid g_i(x)=0\},
\end{equation}
and $g_i$ is affine. For $\boldsymbol{\eta}\in\{\pm 1\}^m$, define the (open) cell
\begin{equation}
\mathcal{C}^\circ(\boldsymbol{\eta})
\;\triangleq\;
P \cap \bigcap_{i=1}^{m} \left\{ x \ \middle|\ \eta_i\,g_i(x) > 0\right\}.
\end{equation}

\medskip
\noindent\textbf{(1) Cells coincide with the connected components of $U$.}
Define the hyperplane-removed set
\begin{equation}
U \;\triangleq\; P\setminus \bigcup_{i=1}^{m}(h_i\cap P).
\end{equation}
For any $x\in U$, we have $g_i(x)\neq 0$ for all $i$, hence we can define the sign vector
\begin{equation}
\boldsymbol{\eta}(x)\in\{\pm 1\}^m,
\qquad
\eta_i(x)\triangleq \operatorname{sgn}(g_i(x)).
\end{equation}
By construction, $x\in\mathcal{C}^\circ(\boldsymbol{\eta}(x))$, and different sign vectors produce disjoint sets.
Therefore,
\begin{equation}
U \;=\; \bigsqcup_{\boldsymbol{\eta}\in\{\pm 1\}^m}\mathcal{C}^\circ(\boldsymbol{\eta}),
\end{equation}
where the union ranges effectively over the non-empty cells.

Each $\mathcal{C}^\circ(\boldsymbol{\eta})$ is the intersection of $P$ with strict halfspaces, hence it is convex.
Consequently, whenever $\mathcal{C}^\circ(\boldsymbol{\eta})$ is non-empty, it is connected.

Now let $\gamma:[0,1]\to U$ be any continuous path. For each $i$, the scalar function
$t\mapsto g_i(\gamma(t))$ is continuous and never equals $0$ on $[0,1]$, hence its sign is constant on
$[0,1]$. Thus the sign vector $\boldsymbol{\eta}(\gamma(t))$ is constant along $\gamma$, implying that no
connected subset of $U$ can intersect two different cells. Therefore, each connected component of $U$ is
contained in a single cell, and each non-empty cell is contained in a connected component.

Combining the two inclusions above and using that each non-empty $\mathcal{C}^\circ(\boldsymbol{\eta})$ is
connected, it follows that the non-empty cells $\mathcal{C}^\circ(\boldsymbol{\eta})$ are exactly the connected
components of $U$.

Finally, $\bigcup_i(h_i\cap P)$ is a finite union of intersections of $P$ with affine hyperplanes, hence it has
Lebesgue measure zero in $\mathbb{R}^{d_0}$. Therefore, the full-dimensional cells partition $P$ up to a boundary
set of measure zero.

\medskip
\noindent\textbf{(2) Connectivity of the adjacency graph and BFS completeness.}
Let $G$ be the adjacency graph whose nodes are the non-empty full-dimensional cells, and where two nodes are
connected if the closures of the corresponding cells share a common facet contained in $h_i\cap P$ for some $i$.

Take any two non-empty full-dimensional cells $C_1=\mathcal{C}^\circ(\boldsymbol{\eta})$ and
$C_2=\mathcal{C}^\circ(\boldsymbol{\xi})$, and choose points $x\in C_1$ and $y\in C_2$.
Consider the line segment
\begin{equation}
\gamma(t)\;=\;(1-t)x+ty,\qquad t\in[0,1].
\end{equation}
By convexity of $P$, $\gamma(t)\in P$ for all $t$.

For each $i$, the function $t\mapsto g_i(\gamma(t))$ is affine in $t$, hence it either has at most one zero in
$(0,1)$ or is identically zero. Since $x,y\in U$, we have $g_i(x)\neq 0$ and $g_i(y)\neq 0$, so it cannot be
identically zero. Therefore, the set of crossing parameters
\begin{equation}
S \;\triangleq\; \{t\in(0,1)\mid \exists i,\ g_i(\gamma(t))=0\}
\end{equation}
is finite.

If at some $t^\star\in S$ the segment intersects two or more hyperplanes simultaneously, then
$\gamma(t^\star)$ lies in a lower-dimensional intersection of hyperplanes. Choose $y$ in the relative interior
of $C_2$; then for sufficiently small perturbations within $\mathrm{aff}(P)$ there exists $y'\in C_2\subseteq P$
such that the segment from $x$ to $y'$ intersects the hyperplanes transversely, i.e., it crosses one hyperplane
at a time and never hits an intersection of codimension at least $2$. Replacing $y$ by $y'$ does not change the
target cell, so we may assume without loss of generality that the segment crosses the hyperplanes one-by-one.

Under this genericity assumption, as $t$ increases from $0$ to $1$, the sign vector of
\[
\big(g_1(\gamma(t)),\dots,g_m(\gamma(t))\big)
\]
changes by flipping exactly one coordinate at each crossing.
Hence the segment passes through a finite sequence of full-dimensional cells
\begin{equation}
C_1 = D_0, D_1, \dots, D_K = C_2,
\end{equation}
where consecutive cells $D_{k-1}$ and $D_k$ are separated by a facet lying in $h_{i_k}\cap P$ for some $i_k$.
By definition of $G$, this yields an edge between $D_{k-1}$ and $D_k$. Therefore there exists a path in $G$
from $C_1$ to $C_2$, proving that $G$ is connected.

Since $G$ is connected, a breadth-first search starting from any node and repeatedly visiting all feasible,
previously unseen neighbors across shared facets will eventually visit every node of $G$. Thus, starting from
any interior point in $P$ and repeatedly crossing feasible hyperplane-induced facets enumerates all non-empty
full-dimensional cells in $P$.
\end{proof}

We introduce a subset $\widehat{\mathcal{H}}_l^{\hat{x}}$ to represent the collection of hyperplanes intersecting with $A_{l-1}^{\hat{x}}$. The complement of $\widehat{\mathcal{H}}_l^{\hat{x}}$ corresponds to neurons that remain either entirely activated or deactivated for any input $x\in A_{l-1}^{\hat{x}}$ (Fig.~\ref{fig:f2}(b)). Consequently, we focus on the neurons in an intermediate state and seek to identify subregions formed by the intersections of these hyperplanes and the parent region. Our approach begins with an initial input space $A_{0}$ and a point $x$ within it. Using Alg.~\ref{alg:search-sub-regions}, we compute the hyperplane set $\mathcal{H}_1^{\hat{x}}$ formed by the neurons in the first layer of a given PANN over the input space $A_{0}$. To search for the subset $\widehat{\mathcal{H}}_1^{\hat{x}}$, we reformulate the task as a linear programming problem. Thus, all subregions can be represented by $\widehat{\mathcal{H}}_1^{\hat{x}}$ and $A_{0}$, effectively reducing the search space for sub-regions. Since $x$ belongs to the parent region, there must exist a sub-region $A_{1}^{x}$ that belongs to $A_{0}$. Using this sub-region as the starting point for a breadth-first search, we compute the representative $\hat{x}_i$ for this region and explore adjacent subregions along its edges. We also reformulate the edge search as a linear programming problem, as before. Through Alg.~\ref{alg:search-sub-regions}, we obtain the complete set of sub-regions of $A_{0}$, and repeat this for all layers up to $l=L$. We provide pseudocode for the algorithm, which is also applicable to other neural network structures that can be represented through affine functions. For instance, a 2D convolution can be written as a sparse linear map, hence as a special case of an affine layer after vectorization \cite{book34}.

\paragraph{Layer-by-layer enumeration (Alg.~\ref{alg:find-cpas})}
Alg.~\ref{alg:find-cpas} is the outer orchestration procedure that enumerates all final-layer CPAs on the calibrated input polytope $A_0$ by repeatedly invoking Alg.~\ref{alg:search-sub-regions} in a layer-wise manner.
It maintains a queue $Q$ of tuples $(l, x_i, A_l^{\hat{x}_i})$, where $l$ is the current depth, $A_l^{\hat{x}_i}$ is a discovered layer-$l$ affine region in input space, and $x_i$ is an interior representative point of this region.
Starting from $(0, x, A_0)$ with any $x\in A_0$, the algorithm pops one tuple at a time.
If $l=L$, the region is already a maximal CPA region and is appended to the output set $E_c$ as $(\hat{x}_i, A_L^{\hat{x}_i})$.
Otherwise, it calls \texttt{SearchSubRegions} on the current parent region $A_l^{\hat{x}_i}$ to enumerate all non-empty child regions $\{A_{l+1}^{\hat{x}_j}\}$ induced by layer $l\!+\!1$, each paired with a newly computed interior representative $\hat{x}_j$.
These children are then pushed back into $Q$ as $(l+1,\hat{x}_j,A_{l+1}^{\hat{x}_j})$, which corresponds to advancing one layer deeper. Importantly, \texttt{SearchSubRegions} canonicalizes region constraints (e.g., via redundancy removal) and only returns previously unseen regions, ensuring that each layer-$l$ region is explored at most once.
Consequently, Alg.~\ref{alg:find-cpas} performs an exhaustive traversal of the implicit region tree from $A_0$ to depth $L$, and returns all non-empty maximal affine regions on $A_0$.

\subsection{Correctness of AffineLens}
\label{sec:correctness}

We give a concise justification for the correctness of Alg.~\ref{alg:search-sub-regions} and
Alg.~\ref{alg:find-cpas}. We consider full-dimensional regions; boundary cases where some pre-activation equals
$0$ have measure zero and do not affect the CPA partition.

\paragraph{Key fact (finite polyhedral partition)}
For CPA networks, under any fixed sign pattern, the network (and each layer pre-activation) is affine.
Thus each layer-$l$ region $A_l^{\hat{x}}$ is defined by finitely many linear inequalities (Eq.~\eqref{(10)}),
hence is a polytope. Intersecting a parent polytope with finitely many neuron hyperplanes produces a finite
set of polyhedral sub-regions.

\paragraph{Soundness}
Alg.~\ref{alg:search-sub-regions} only outputs regions constructed from explicit linear constraints:
(i) the parent constraints (from $A_l^{\hat{x}}$) and (ii) a consistent set of sidedness constraints induced by
hyperplanes in $\widehat{\mathcal{H}}$ (those intersecting the parent region).
Each time a “new region” is added, it is obtained via a feasibility check (LP) and therefore is non-empty.
Hence every returned $(\hat{x}_i, A_{l+1}^{\hat{x}_i})$ is indeed a valid sub-region of the parent and corresponds
to a valid activation pattern at layer $l\!+\!1$.

\paragraph{Completeness}
Inside a convex parent region, the intersecting hyperplanes $\widehat{\mathcal{H}}$ induce a subdivision into
cells whose adjacencies occur across shared facets (hyperplane boundaries).
Starting from the cell containing the initial point $x$, the algorithm explores neighbor cells by crossing
facets and adding any feasible, previously unseen cell into the BFS queue.
Because the adjacency graph of cells within a convex polytope is connected, BFS visits all cells.
Therefore Alg.~\ref{alg:search-sub-regions} enumerates all sub-regions of the parent region.
Alg.~\ref{alg:find-cpas} applies this argument layer-by-layer, so all final-layer (maximal) affine regions on
$A_0$ are eventually discovered.

\paragraph{Termination}
At each layer, a bounded polytope intersected with finitely many hyperplanes yields only finitely many
full-dimensional cells. The algorithms maintain queues/sets and insert a region only when it has not been
seen before. Hence only finitely many regions can be processed at every layer, and both BFS procedures
terminate.

\subsection{Complexity of AffineLens}
\label{sec:complexity}

We analyze the worst-case complexity of AffineLens on a bounded calibrated input polytope $A_0\subset\mathbb{R}^{d_0}$.
Unless stated otherwise, we consider full-dimensional regions; boundary degeneracies have measure zero.

\paragraph{Per-layer combinatorics}
Fix a parent region in $\mathbb{R}^{d_0}$ and consider the arrangement induced by $m_l$ hyperplanes restricted to that bounded convex set.
The number of full-dimensional cells is at most
\begin{equation}
N_l(m_l,d_0)\;\le\;\sum_{i=0}^{d_0}\binom{m_l}{i}
\;=\;O\!\big((m_l+1)^{d_0}\big),
\label{eq:Nl}
\end{equation}
and therefore $N_l(d_l,d_0)$ upper bounds the worst case when $m_l$ is replaced by $d_l$.
Moreover, when $m_l\ge 1$ and $d_0$ is fixed, $\sum_{i=0}^{d_0}\binom{m_l}{i}=\Theta(m_l^{d_0})$.

\paragraph{LP-call counts inside Alg.~\ref{alg:search-sub-regions} (per parent region)}
For a fixed parent region $A_{l-1}^{\hat{x}}$, Alg.~\ref{alg:search-sub-regions} performs LP-based operations:
\begin{itemize}
\item \textbf{Hyperplane--region intersection filtering:} $O(d_l)$ feasibility (or min/max) LPs to determine $\widehat{\mathcal{H}}_l^{\hat{x}}$ (whose size is $m_l\le d_l$).
\item \textbf{Per-cell representative:} $O(N_l)$ Chebyshev-center LPs (one per enumerated cell). These LPs use $d_0{+}1$ variables (radius included).
\item \textbf{Neighbor discovery:} $O(m_l\,N_l)$ facet-crossing / neighbor-feasibility LPs (each cell has $O(m_l)$ candidate facets induced by $\widehat{\mathcal{H}}_l^{\hat{x}}$).
\end{itemize}
Optional redundancy removal can be treated as post-processing; if implemented via full LP-based redundancy checks,
it adds an extra $O(K_l\,N_l)$ LPs, but does not change the core enumeration logic.
Hence, the dominant LP count per parent region is
\begin{equation}
\begin{split}
\#\mathrm{LPs\ per\ parent\ region} 
&= O\!\big(m_l\,N_l(m_l,d_0)\big)\\
&= O\!\big(m_l\,(m_l+1)^{d_0}\big),
\end{split}
\label{eq:LPcount}
\end{equation}
and in the worst case $\#\mathrm{LPs}=O(d_l^{\,d_0+1})$.

\paragraph{Constraint growth}
If $K_{l-1}$ inequalities describe a parent region at layer $l{-}1$, then each child region at layer $l$
adds at most one sidedness constraint per intersecting hyperplane, i.e., at most $m_l$ new inequalities. Thus
\begin{equation}
K_l
\;=\;
O\!\big(K_{l-1}+m_l\big)
\;=\;
O\!\Big(K_0+\sum_{t=1}^{l} m_t\Big),
\label{eq:Kl}
\end{equation}
and in worst-case bounds (without redundancy removal) $K_l=K_0+\sum_{t=1}^{l} d_t$.

\paragraph{Per-layer time}
Let $R_l$ be the number of distinct regions after layer $l$ (so $R_0=1$).
Each region at layer $l{-}1$ spawns at most $N_l(m_l,d_0)$ children, hence
\begin{equation}
\begin{aligned}
R_l &\;\le\; R_{l-1}\,N_l(m_l,d_0) \\
\Rightarrow\quad R_{l-1} &\;\le\; \prod_{t=1}^{l-1} N_t(m_t,d_0).
\end{aligned}
\end{equation}
Combining with \eqref{eq:LPcount}, the total time spent at layer $l$ is bounded by
\begin{equation}
\begin{split}
T_l = O\Big(&
R_{l-1}\, m_l\,(m_l+1)^{d_0} \; \times \\
& \max\big\{T_{\mathrm{LP}}(d_0,K_l),\,T_{\mathrm{LP}}(d_0+1,K_l)\big\}
\Big).
\end{split}
\label{eq:Tl}
\end{equation}
For simplicity, one may write $T_{\mathrm{LP}}(d_0{+}1,K_l)$ as the dominant term.

\paragraph{Total time and a compact upper bound}
Summing \eqref{eq:Tl} across layers and using $N_t(m_t,d_0)=O((m_t+1)^{d_0})$ gives the explicit worst-case bound
\begin{equation}
T_{\mathrm{total}}
\;=\;
\sum_{l=1}^{L}
\Big(\prod_{t=1}^{l-1} (m_t+1)^{\,d_0}\Big)\,
m_l\,(m_l+1)^{\,d_0}\,
T_{\mathrm{LP}}(d_0{+}1,K_l).
\label{eq:Tsum-start}
\end{equation}
Moreover, to avoid degenerate cases where $m_t=0$, define $\bar m_t\triangleq \max\{1,m_t\}$ and
$P\triangleq\prod_{t=1}^L \bar m_t^{d_0}$. Then,
\begin{equation}
\Big(\prod_{t=1}^{l-1} \bar m_t^{\,d_0}\Big)\,\bar m_l^{\,d_0+1}
= P \cdot \frac{\bar m_l}{\prod_{t=l+1}^L \bar m_t^{\,d_0}}
\;\le\; P \, \bar m_l.
\end{equation}
and therefore a looser but compact bound is
\begin{equation}
T_{\mathrm{total}}
\;=\;
O\!\Big(
\big[\prod_{t=1}^{L} \bar m_t^{\,d_0}\big]\,
\big[\sum_{l=1}^{L} \bar m_l\big]\,
\max_{1\le l\le L} T_{\mathrm{LP}}(d_0{+}1,K_l)
\Big).
\label{eq:Ttotal-compact}
\end{equation}
In worst-case notation, replace $\bar m_l$ by $d_l$.

\paragraph{Exponent-style simplification}
Let $\bar m_{\max}=\max_{1\le t\le L} \bar m_t$ and $\bar m_{\mathrm{geo}}=(\prod_{t=1}^L \bar m_t)^{1/L}$.
Then $\prod_{t=1}^L \bar m_t^{d_0}=\bar m_{\mathrm{geo}}^{d_0 L}$ and $\sum_{l=1}^L \bar m_l\le L\,\bar m_{\max}$, so from \eqref{eq:Ttotal-compact},
\begin{equation}
T_{\mathrm{total}}
\;=\;
O\!\Big(
\bar m_{\mathrm{geo}}^{\,d_0 L}\,
L\,\bar m_{\max}\,
\max_{1\le l\le L} T_{\mathrm{LP}}(d_0{+}1,K_l)
\Big).
\end{equation}
For uniform width and full intersection ($m_l\equiv d$ so $\bar m_l\equiv d$), we obtain
\begin{equation}
T_{\mathrm{total}}
\;=\;
O\!\Big(
d^{\,d_0 L + 1}\,L\,T_{\mathrm{LP}}(d_0{+}1,\,K_0+Ld)
\Big).
\end{equation}

\section{Experiments}
\label{sec:experiments}
We conduct 6 groups of experiments to demonstrate the capability of AffineLens in computing, enumerating, and visualizing affine regions across diverse PANN architectures. Our focus is empirical: we aim to reveal how architectural choices---including width, depth, neuron composition, convolutional inductive biases, and residual connections---as well as training dynamics and activation functions, shape both the number and the geometry of affine regions expressed by trained networks. Concretely, we study (i) the effect of network width on region counts, (ii) the effect of depth on region counts, (iii) how neurons in shallow vs.\ deep PANNs contribute to region expressivity, (iv) how convolutional architectures organize affine-region partitions compared with MLPs, (v) how the number of affine regions varies in residual networks with different numbers of residual units, and (vi) how affine regions and decision boundaries evolve during training under different activation functions (ReLU vs.\ LeakyReLU).

\subsection{Experimental Setup}
\label{sec:4.1}
The data distributions for the two-moons dataset and the synthetic random dataset used throughout the paper are shown in Fig.~\ref{fig:f6}. Unless otherwise specified, all experiments use ReLU activations; in addition, we include LeakyReLU as a complementary comparison in~\ref{sec:4.5} and~\ref{sec:4.7}. Affine-region counts depend on the input domain. In all experiments, we use an $\ell_{\infty}$-box (hypercube) calibration domain
\begin{equation}
A_{0}=\{x\in\mathbb{R}^{d_{0}}\mid \|x\|_{\infty}\le 1\}=[-1,1]^{d_{0}}.
\end{equation}
Specifically, we use $A_{0}=[-1,1]^2$ for the 2D experiments, $A_{0}=[-1,1]^3$ for the 3D experiments, and $A_{0}=[-1,1]^5$ for the 5D experiments. For initialization, we use Kaiming-uniform weights to stabilize gradient propagation under CPA activations. All models are trained with Adam (batch size 32) using cross-entropy loss and a fixed learning rate of $10^{-4}$; unless otherwise stated, other hyperparameters follow default settings. Our implementation is based on PyTorch, with NumPy and SciPy used for numerical computation. Because affine-region enumeration becomes computationally intensive for deep PANNs---due to repeated hyperplane filtering and region expansion---we use multi-processing to parallelize core computations on multi-core CPUs. The framework supports user-defined input ranges, enabling localized analysis over arbitrary subdomains when needed. All experiments are conducted on a computing platform with eight NVIDIA RTX 6000 Ada Generation GPUs (48GB VRAM each) and dual AMD EPYC 9754 processors.

\begin{figure}[t]
  \centering
  \subfloat[]{%
    \safeincludegraphics[width=0.42\linewidth,height=12.0cm,keepaspectratio]{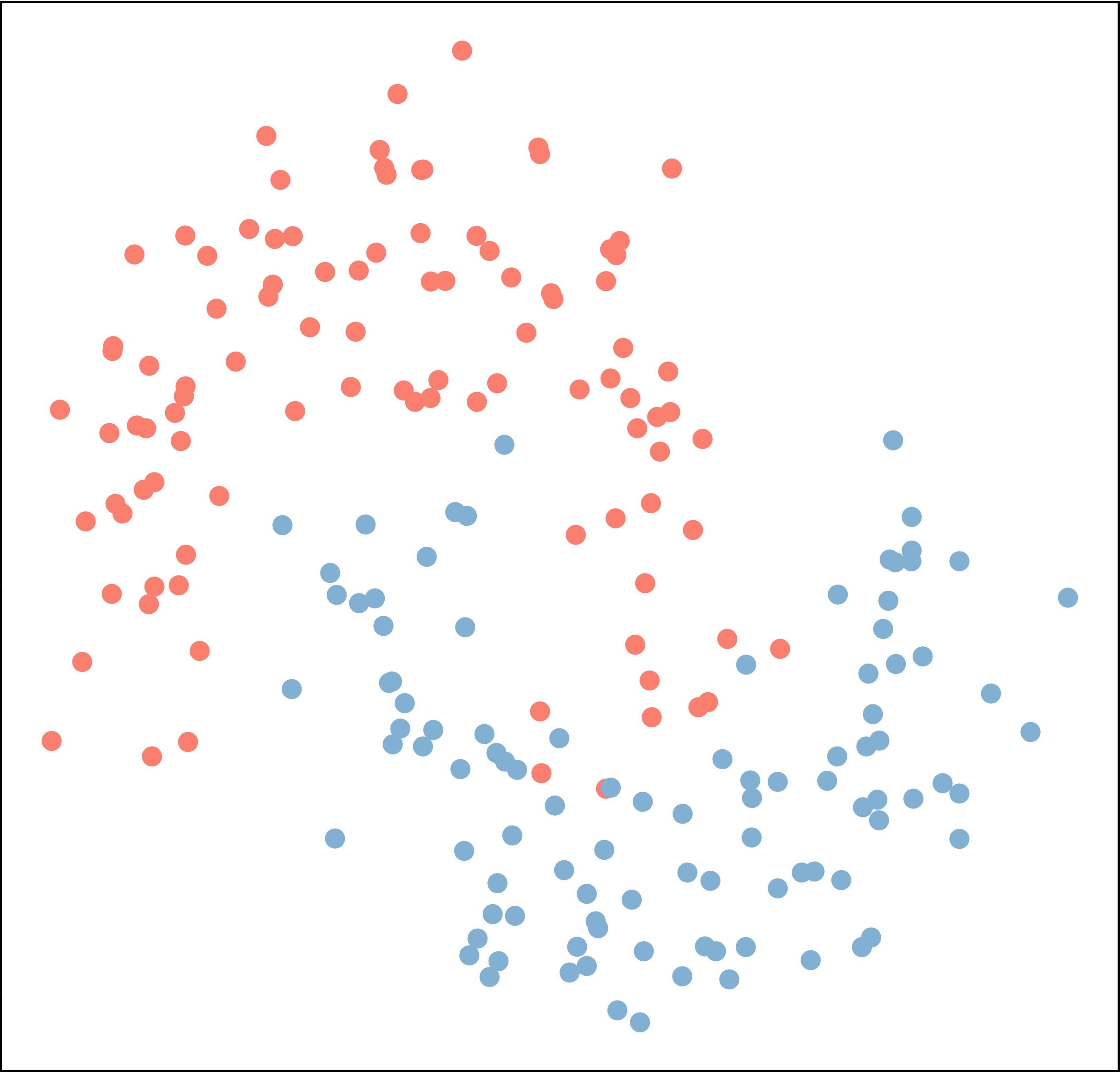}%
    \label{fig:f6a}%
  }\quad
  \subfloat[]{%
    \safeincludegraphics[width=0.42\linewidth,height=6.15cm,keepaspectratio]{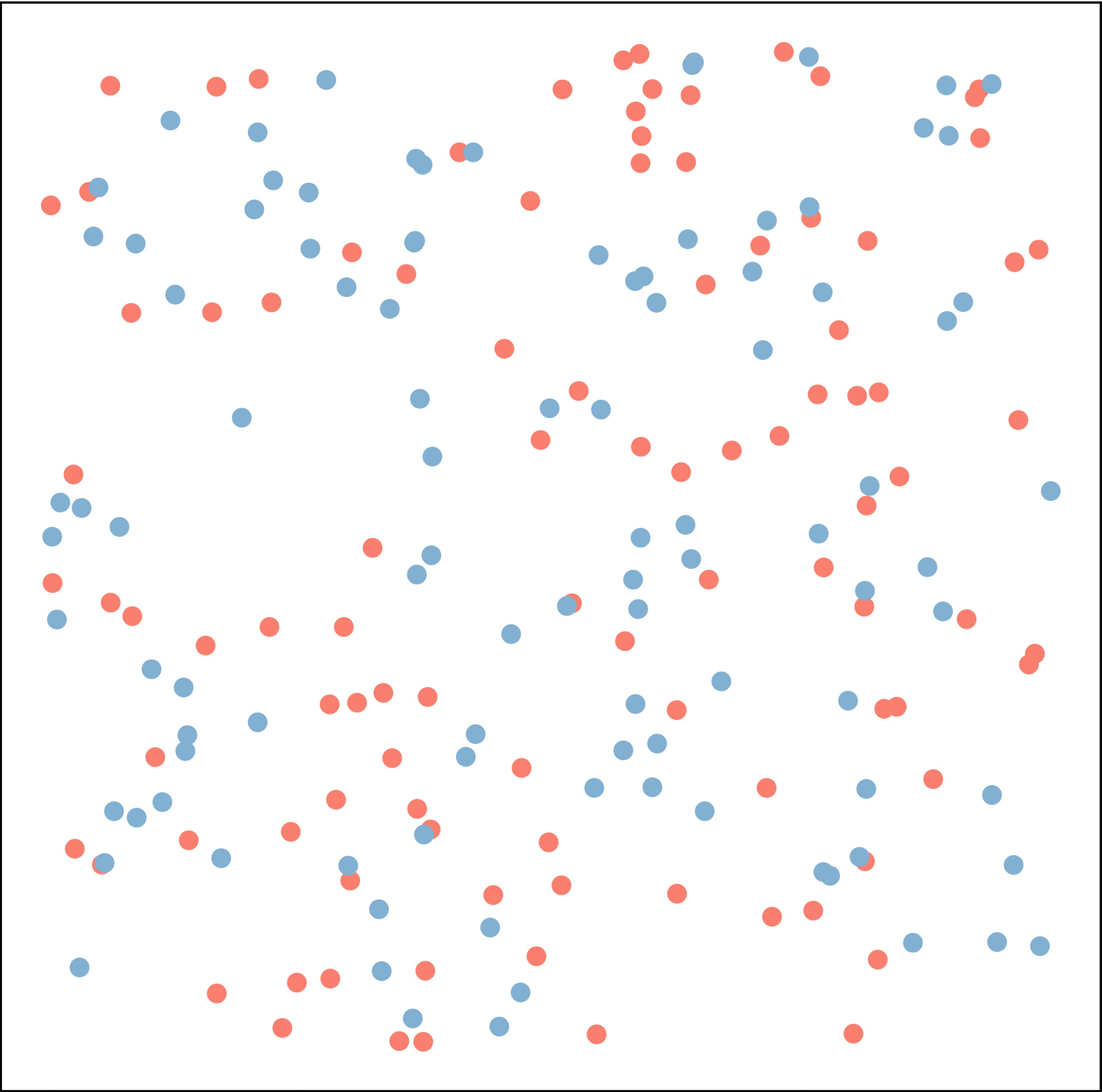}%
    \label{fig:f6b}%
  }

  \caption{Two types of data distributions are used in the experiments, each comprising 200 samples. (a) The two-moon dataset. (b) The synthetic random dataset.}
  \label{fig:f6}
\end{figure}

\subsection{The Impact of Network Width}
\label{sec:4.2}
This experiment investigates the influence of network width on the expressive capacity of PANNs, as quantified by the number of affine regions a network can represent. A fundamental characteristic of PANNs is their ability to partition the input space into multiple linear regions, which directly reflects the model’s capacity to approximate complex functions. However, accurately quantifying the number of affine regions expressed by network width in high-dimensional input spaces remains an open and pressing challenge. A quantitative examination of how width alone contributes to the partitioning structure of PANNs is both timely and necessary, especially for applications that prioritize shallow architectures due to interpretability or computational constraints. To isolate the impact of width, we designed a series of controlled experiments using single-layer MLPs, thus removing any confounding factors introduced by varying depth. Each network was trained using the synthetic random dataset, initialization schemes, and hyperparameter settings; width was the only parameter allowed to vary, with values set to 16, 32, and 64 neurons, respectively. This controlled setup ensures that any observed differences in the number of affine regions can be directly attributed to changes in width. We conducted experiments across four different input dimensionalities (2D, 3D, 4D, and 5D) to further examine how the interaction between network width and input space geometry influences expressivity. For each configuration, we computed the number of affine regions formed within a fixed calibration area after 5000 training epochs. The results, summarized in Table \ref{tab:table1}, consistently show that increasing the width leads to a substantial increase in the number of linear regions across all input dimensions. Moreover, a strong dimensional effect is observed: as the input space expands from 2D to 5D, the growth in region count becomes increasingly pronounced, even under the same width settings. This behavior aligns with the geometric intuition that higher-dimensional spaces allow for more intricate hyperplane arrangements, resulting in a combinatorial explosion of partitioned regions. These findings underscore the critical role of width in enhancing the representational power of shallow PANNs. By empirically demonstrating how width influences the granularity of affine partitions, particularly in high-dimensional settings, AffineLens provides valuable insights for architecture design and capacity control in both theoretical and applied contexts.
\begin{table}[htbp]
\centering
\small
\caption{Number of affine regions expressed by PANNs of varying widths across different input dimensionalities.}
\label{tab:table1}
\begin{tabularx}{\linewidth}{l|X X X X}
\toprule
\diagbox{Width}{Dim.} & 2D & 3D & 4D & 5D \\
\midrule
\textnormal{[16]} & 43 & 318 & 852 & 2,684 \\
\textnormal{[32]} & 218 & 1,864 & 10,663 & 62,027 \\
\textnormal{[64]} & 742 & 15,800 & 188,312 & 2,149,855 \\
\bottomrule
\end{tabularx}
\end{table}

\subsection{The Impact of Network Depth}
\label{sec:4.3}
This experiment is designed to quantitatively evaluate the influence of PANN depth on expressive capacity, as reflected by the number of affine regions the network is capable of representing. In particular, we examine how variations in depth affect expressivity under diverse input distributions and dimensional settings. To ensure that the observed effects can be attributed exclusively to depth, the network width is held constant by fixing the number of neurons per layer at 16, 32, and 64, respectively, across all configurations. This approach effectively controls for confounding factors introduced by variations in network width. While prior theoretical studies have suggested that deeper networks can represent increasingly complex piecewise linear functions by inducing a larger number of affine regions, empirical validation of this claim remains limited. Understanding how depth interacts with other factors such as input dimensionality and distribution is crucial for both the theoretical characterization and practical design of deep architectures. We conducted experiments using 2 distinct types of input distributions: the two moon dataset \cite{book44} and a synthetic random dataset. All networks were trained using identical input samples, weight initialization schemes, and hyperparameter settings to ensure experimental consistency. The only variable modified was network depth, which was varied systematically from 2 to 10 layers. After 5000 epochs of training, we measured the number of distinct affine regions represented by the network within a fixed calibration area. As illustrated in Fig. \ref{fig:f3}, deeper networks are capable of partitioning the input space into a greater number of affine regions. Furthermore, Table \ref{tab:table2} reports experimental results combining increased input dimensionality with varying network depths. When the input space is expanded from 2D to 4D, the number of affine regions grows more rapidly with increasing depth. This observation highlights a notable increase in the combinatorial complexity of hyperplane arrangements in higher-dimensional spaces, further emphasizing the role of depth in enhancing representational power. Taken together, these findings provide strong empirical evidence that increasing the depth of a PANN significantly boosts its expressive capacity by partitioning the input space into finer-grained affine regions.
\begin{figure}[htbp]
\centering
\safeincludegraphics[width=\linewidth,height=0.42\textheight,keepaspectratio]{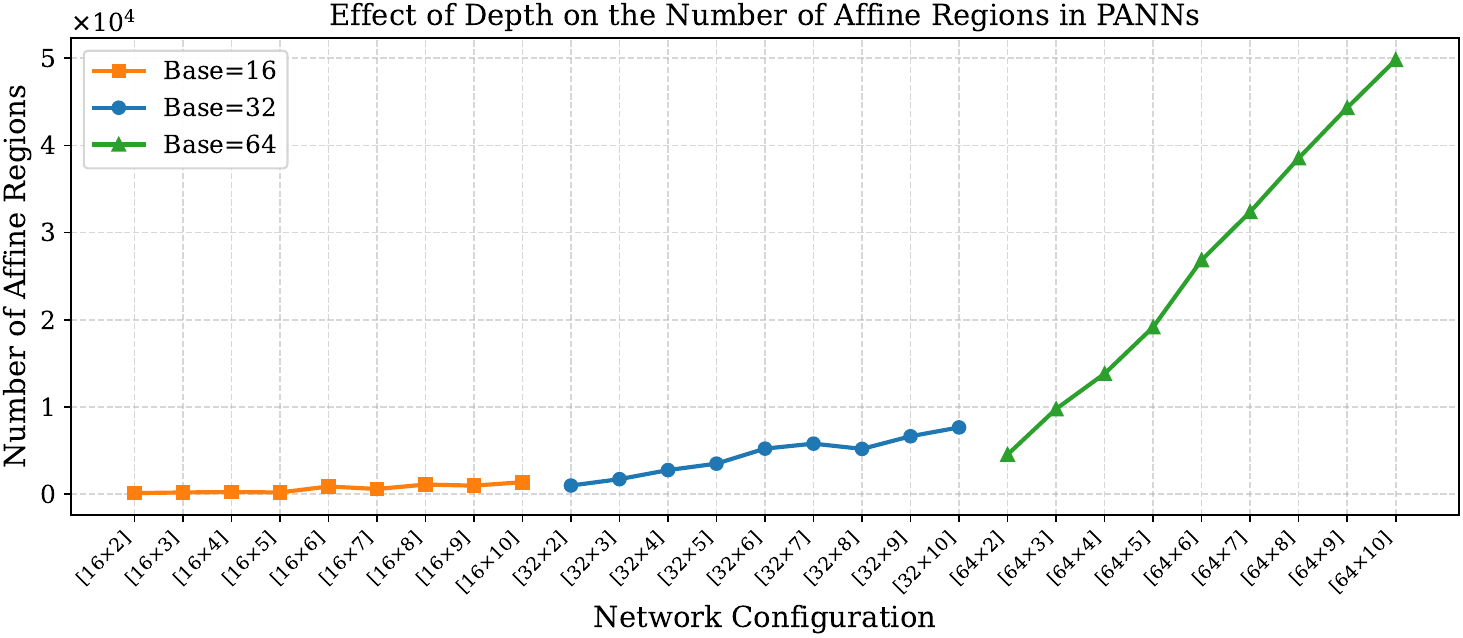}
\caption{The impact of network depth on the expressivity in terms of the number of representable affine regions.}
\label{fig:f3}
\end{figure}

\begin{table}[htbp]
\centering
\small
\caption{Number of affine regions under different dimensions and depth.}
\label{tab:table2}

\textbf{(a) Configurations with 16 neurons per layer}\\
\begin{tabularx}{\linewidth}{c|XXX}
\toprule
\diagbox{Dim.}{Depth} & [16$\times$2] & [16$\times$3] & [16$\times$4] \\
\midrule
3D & 1,077  & 3,505  & 10,524  \\
4D & 11,937 & 58,589 & 163,502 \\
\bottomrule
\end{tabularx}

\vspace{6pt}

\textbf{(b) Configurations with 32 neurons per layer}\\
\begin{tabularx}{\linewidth}{c|XXX}
\toprule
\diagbox{Dim.}{Depth} & [32$\times$2] & [32$\times$3] & [32$\times$4] \\
\midrule
3D & 12,442  & 96,551  & 102,922  \\
4D & 256,653 & 929,162 & 2,294,901 \\
\bottomrule
\end{tabularx}
\end{table}

\subsection{The Impact of Shallow and Deep Layers}
\label{sec:4.4}
This experiment aims to quantitatively investigate the distinct contributions of shallow and deep layers within PANNs to the number of affine regions that the network can represent. Understanding how individual layers affect the network’s overall expressive capacity is essential, especially when the total neuron count remains fixed. Such knowledge is critical for the design and optimization of neural architectures, as it informs how computational resources should be allocated across layers to maximize representational power without increasing overall complexity. To isolate and analyze the effects of depth and width on expressivity, we performed three controlled experiments based on a consistent five-layer network architecture, where each layer initially contained 32 neurons. In each experiment, a single layer was selectively expanded to a higher neuron count of 64, 128, or 256 neurons, while all other layers remained fixed at 32 neurons. This modification was applied to one layer at a time, progressing sequentially from the shallowest to the deepest layer in the network. By holding all other factors constant, including identical input two moon dataset \cite{book44}, consistent initialization procedures, and uniform hyperparameters, we ensured that any observed differences in expressivity could be attributed solely to the position of the expanded layer within the network’s depth. We monitored the evolution of the number of affine regions formed within a fixed calibration domain at multiple epoch milestones (1, 10, 50, 100, 200, 300, 500, 1000, 2000, and 5000). As shown in Fig. \ref{fig:f4} for the 2D input case, when total neuron count is kept constant, increasing the neuron count in earlier layers leads to a significantly larger number of affine regions during the initial stages of training. Specifically, a comparison of region counts within the first 50 epochs across all three experimental setups shows that, in our experiments, increasing width in earlier layers has a strong effect on region count during early training. While deep layers are often associated with high-level feature extraction, earlier layers may play a more direct role in early-stage partitioning of the input space, which sets the foundation for subsequent transformations. By elucidating how neuron allocation at various depths affects expressivity over time, AffineLens enables a fine-grained investigation into how specific layers influence network expressivity within the architecture. This capability facilitates a more judicious balance between depth and width, thereby enhancing both performance and interpretability.

\begin{figure}[!ht]
  \centering
  \subfloat[Perturbation with a single layer containing 64 neurons.]{%
    \safeincludegraphics[width=\linewidth,height=0.22\textheight,keepaspectratio]{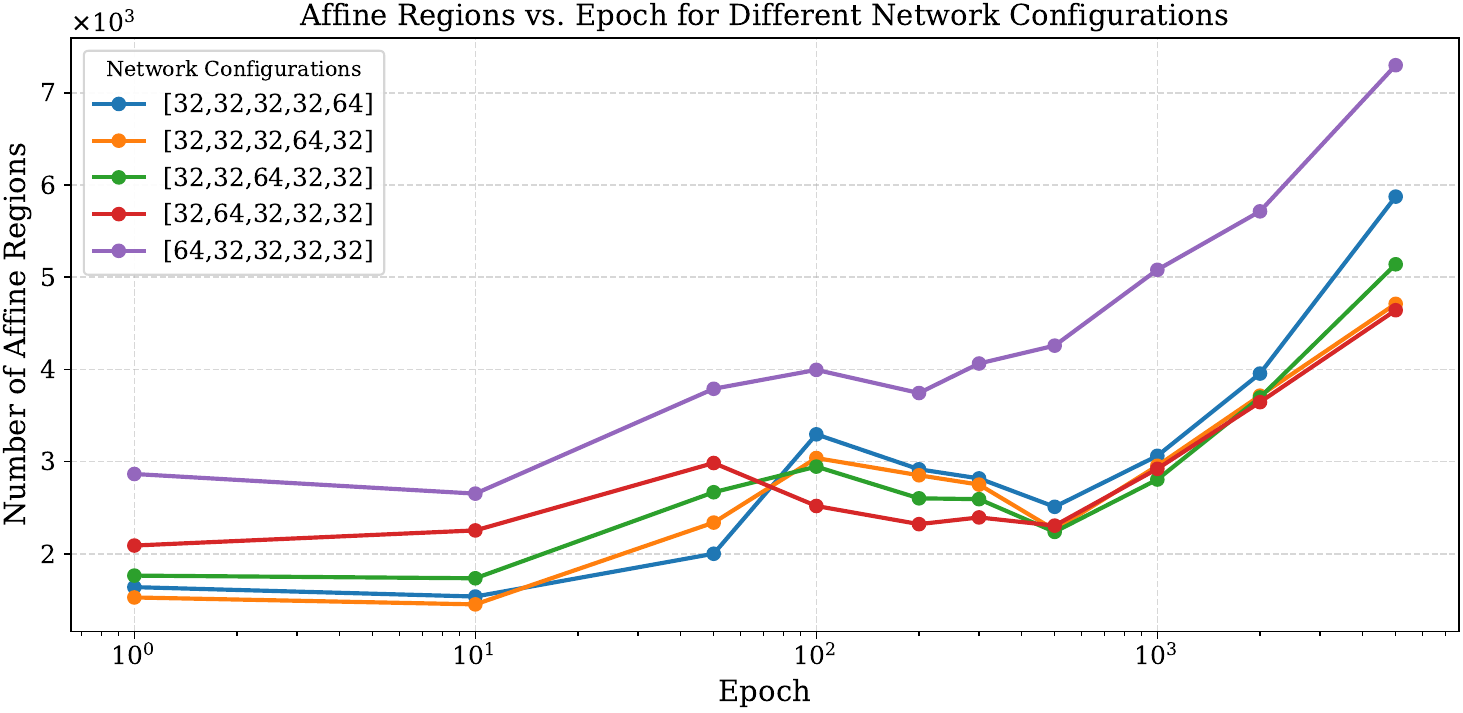}%
    \label{fig:f4a}%
  }\\[-1pt]
  \subfloat[Perturbation with a single layer containing 128 neurons.]{%
    \safeincludegraphics[width=\linewidth,height=0.22\textheight,keepaspectratio]{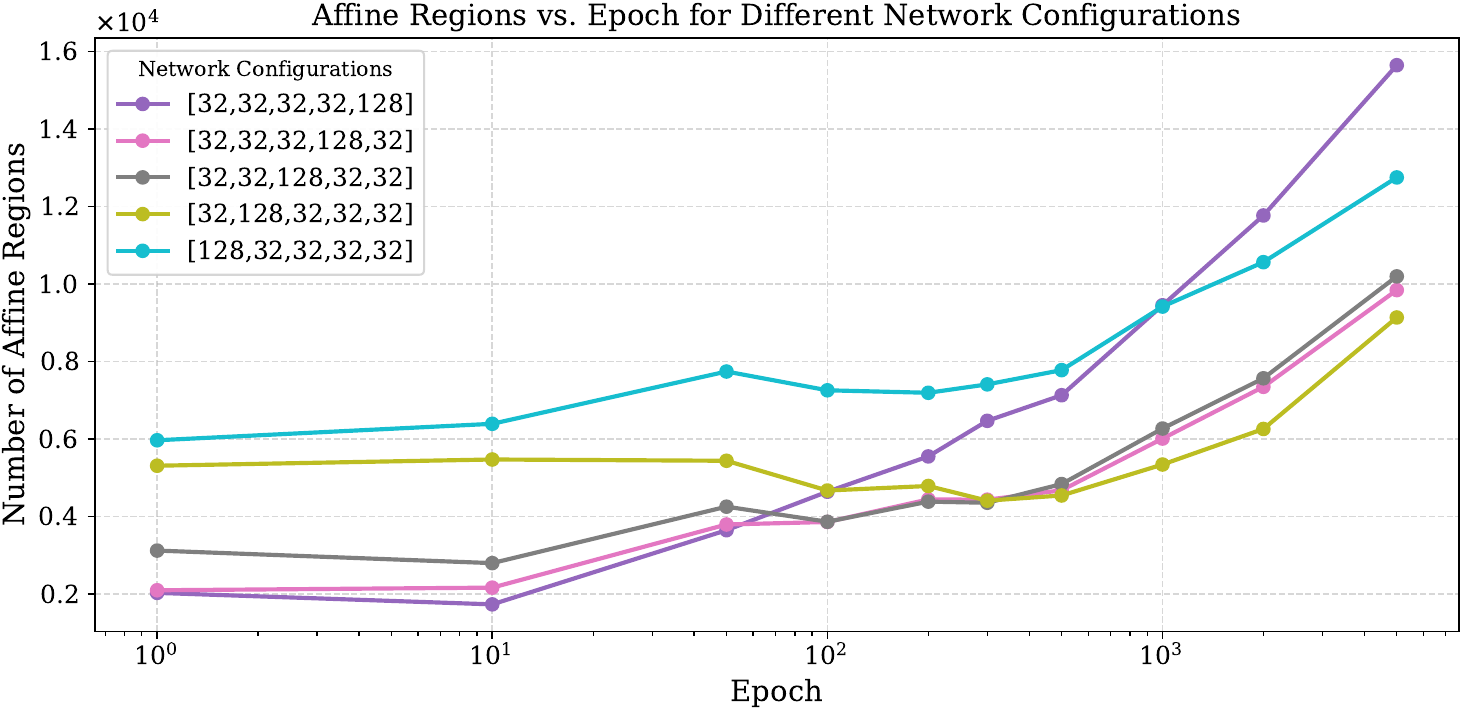}%
    \label{fig:f4b}%
  }\\[-1pt]
  \subfloat[Perturbation with a single layer containing 256 neurons.]{%
    \safeincludegraphics[width=\linewidth,height=0.22\textheight,keepaspectratio]{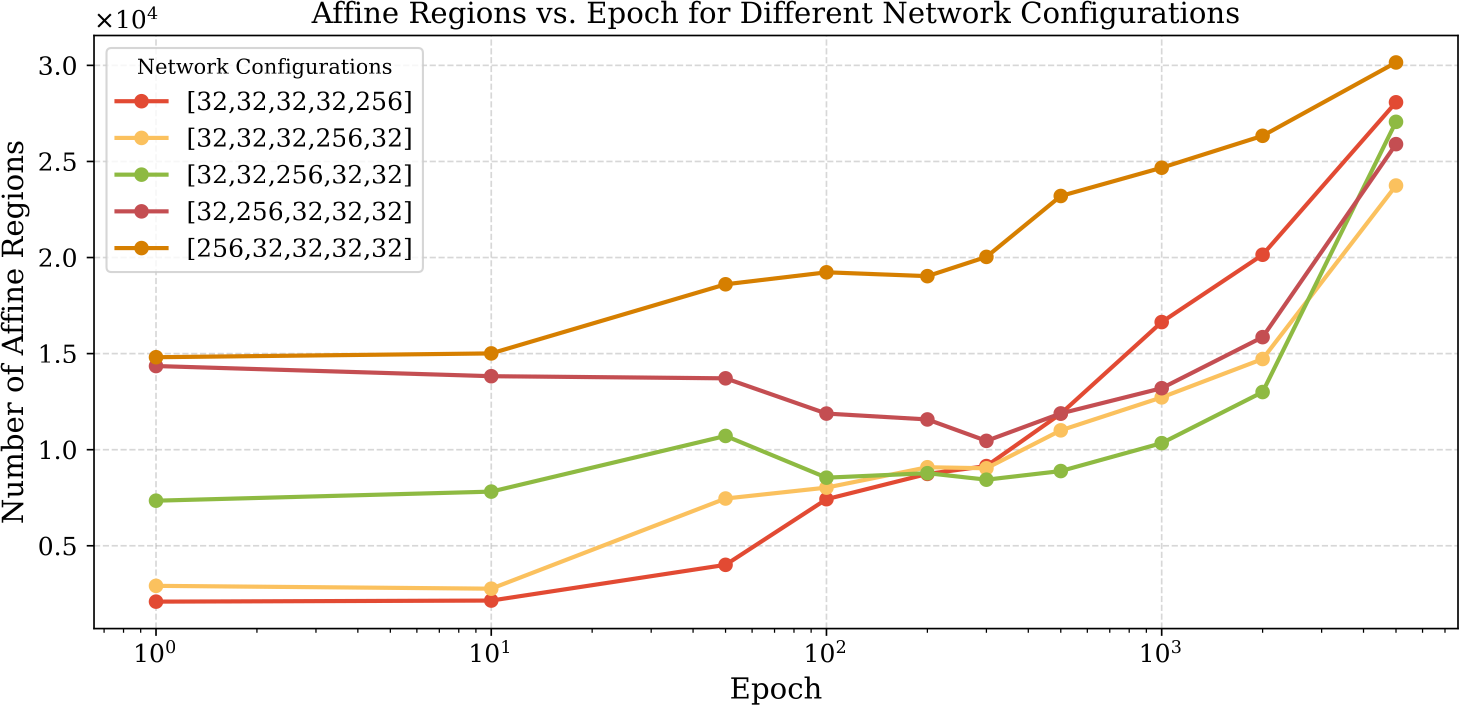}%
    \label{fig:f4c}%
  }

  \caption{An analysis of the influence of neurons in shallow and deep layers on the number of affine regions expressed by PANNs.}
  \label{fig:f4}
\end{figure}

\subsection{The Impact of Convolutional Networks}
\label{sec:4.5}

\begin{figure*}[htbp]
  \centering
  \subfloat[Arrangement and count of affine regions expressed by MLPs.]{%
    \safeincludegraphics[width=0.95\linewidth,height=0.35\textheight,keepaspectratio]{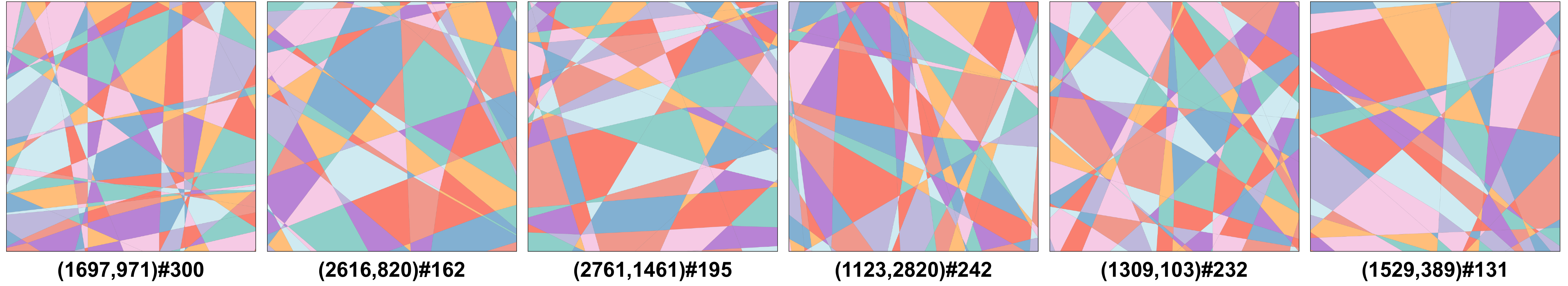}%
    \label{fig:f5a}%
  }\\[3pt]
  \subfloat[Arrangement and count of affine regions expressed by CNNs.]{%
    \safeincludegraphics[width=0.95\linewidth,height=0.35\textheight,keepaspectratio]{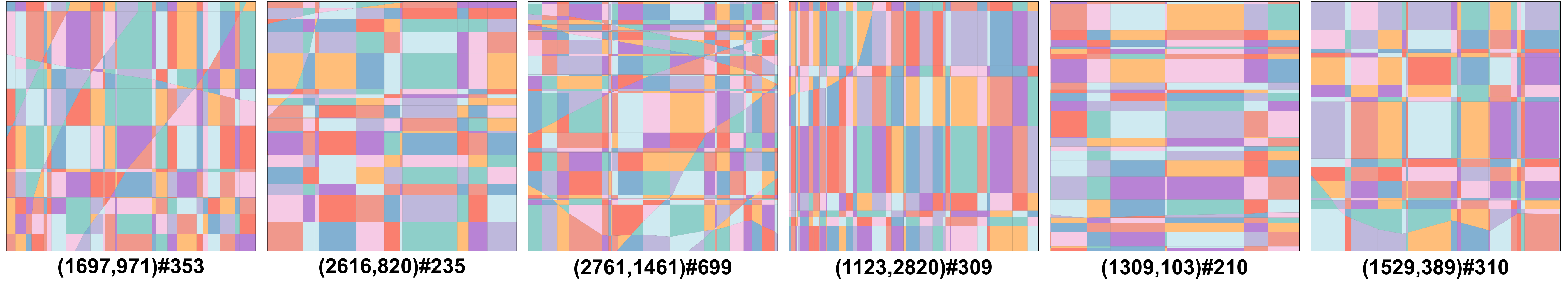}%
    \label{fig:f5b}%
  }

  \caption{Arrangement and quantitative analysis of affine regions expressed by MLPs and CNNs under different 2D projections (ReLU on CIFAR-10).}
  \label{fig:f5}
\end{figure*}

This experiment uses AffineLens to visualize and analyze how convolutional inductive biases shape the organization of affine regions. We perform a controlled comparison between MLPs and CNNs trained on the same task, focusing on differences in the \emph{projected} arrangement and apparent regularity of affine-region patterns. Note that all plots are 2D projections of high-dimensional partitions and are intended as qualitative diagnostics.

\paragraph{ReLU on CIFAR-10 (AlexNet-style CNN).}
We conduct experiments on CIFAR-10 \cite{book45}. For the CNN, we adopt a lightweight AlexNet \cite{book46}-style architecture enhanced with batch normalization and average pooling. All models are trained for 100 epochs using identical optimization settings. After training, we randomly select two input dimensions and visualize affine-region boundaries in the corresponding 2D projection space, and report the number of distinct regions observed in that projection.

Fig.~\ref{fig:f5} shows representative examples. Empirically, CNNs tend to exhibit more regular projected partition patterns (often appearing grid-like), whereas MLPs yield more irregular and unstructured projected arrangements under the same dataset and training protocol. This suggests convolutional priors encourage more spatially coherent functional geometry in low-dimensional projections.

\paragraph{LeakyReLU on MNIST (LeNet-5-style CNN).}
To verify that the above observations are not specific to ReLU, we repeat the same MLP-vs-CNN analysis on MNIST \cite{book47} using LeakyReLU activations. The CNN follows a LeNet-5 \cite{book47}-style architecture augmented with batch normalization and average pooling. Both models are trained for 50 epochs under identical optimization settings. As before, we randomly select two input dimensions for 2D projection, visualize the projected affine-region boundaries, and count the number of distinct projected regions.

\begin{figure*}[htbp]
  \centering
  \subfloat[Arrangement and count of affine regions expressed by MLPs.]{%
    \safeincludegraphics[width=0.95\linewidth,height=0.35\textheight,keepaspectratio]{image/Fig.8a.pdf}%
    \label{fig:f8a}%
  }\\[3pt]
  \subfloat[Arrangement and count of affine regions expressed by CNNs.]{%
    \safeincludegraphics[width=0.95\linewidth,height=0.35\textheight,keepaspectratio]{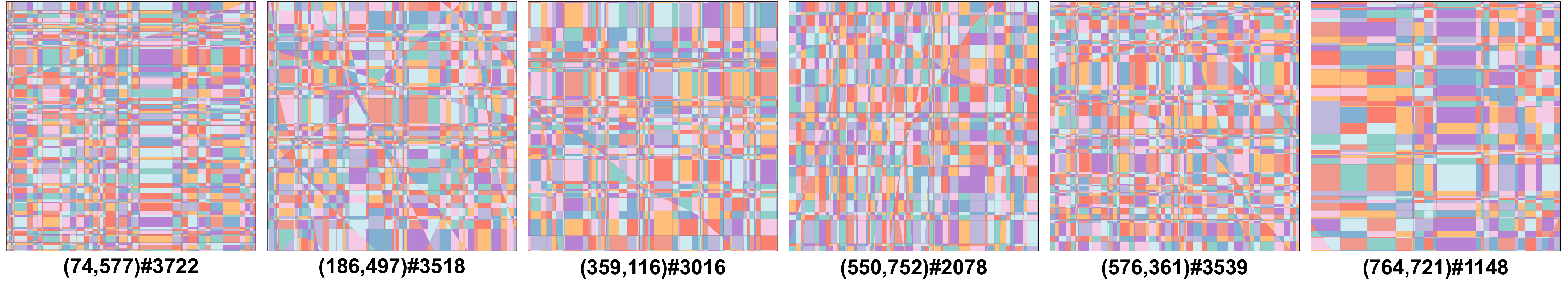}%
    \label{fig:f8b}%
  }

  \caption{Arrangement and quantitative analysis of affine regions expressed by MLPs and CNNs under different 2D projections (LeakyReLU on MNIST).}
  \label{fig:f8}
\end{figure*}

\begin{figure*}[htbp]
  \centering
  \subfloat[]{%
    \safeincludegraphics[width=0.95\linewidth,height=0.31\textheight,keepaspectratio]{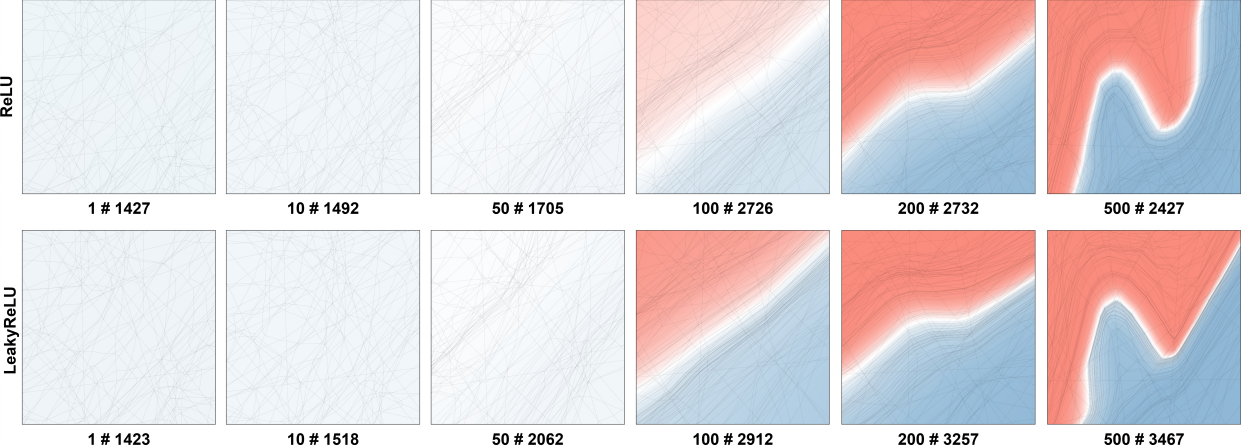}%
    \label{fig:f7a}%
  }\\[3pt]
  \subfloat[]{%
    \safeincludegraphics[width=0.95\linewidth,height=0.31\textheight,keepaspectratio]{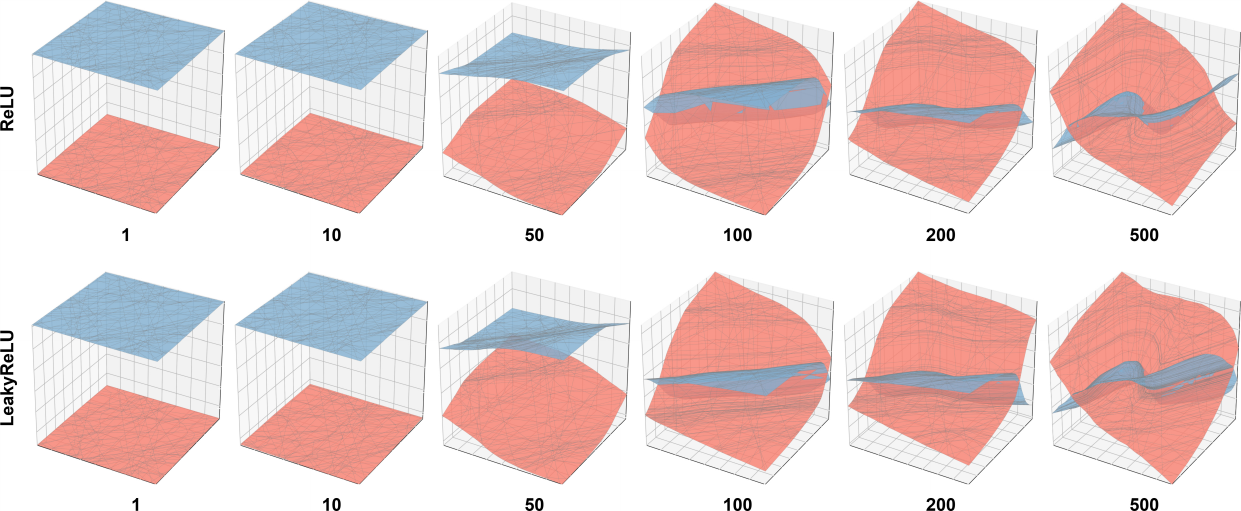}%
    \label{fig:f7b}%
  }

  \caption{(a) Dynamic visualization of decision boundary evolution and affine region formation during training of a 6-layer fully connected MLP on the two-moons dataset. Each layer contains 32 neurons, and ReLU vs.\ LeakyReLU are compared. Visualizations are shown at epochs 1, 10, 50, 100, 200, and 500. Blue and red regions indicate the current affine partitions associated with each class, while the white area approximates the decision boundary inferred from the network output. The number of affine regions at each epoch is reported beneath each plot. (b) 3D visualization of affine region structures and decision boundary dynamics for the same network. Colored hyperplanes represent class-conditional affine segments, and their intersections approximate the evolving decision boundary, providing geometric intuition into how PANNs refine decision boundaries during training.}
  \label{fig:f7}
\end{figure*}

Representative visualizations are presented in Fig.~\ref{fig:f8}. Consistent with the CIFAR-10 results, the CNN again shows more regular projected partition patterns, while the MLP induces more irregular and fragmented projected affine regions. Together, these results indicate that convolutional architecture can shape the observed affine-region organization in low-dimensional projections, and they illustrate that AffineLens can reveal such structural differences across datasets and activation choices.

\subsection{Affine Regions of Residual Connections}
\label{sec:4.6}
We investigate the impact of residual connections on the expressive capacity of PANNs by analyzing the number of affine regions expressed by residual networks with different numbers of skip connections. By quantifying how residual connections alter the number and arrangement of affine regions, we aim to provide new insights into their role in shaping the expressivity of modern deep networks. Specifically, we consider 3 network configurations incorporating 2, 3 and 4 residual connections, respectively. Each network uses hidden width 32 throughout and is trained on the two moon dataset \cite{book44}. Batch Normalization is included in all residual models. Each residual connection is an identity skip within a residual block that adds the block input to the block output. For each setting, we train two variants: a residual version (skip additions enabled) and a non-residual baseline (skip additions removed) while keeping all other components unchanged. To ensure a fair comparison, all experimental settings are kept identical across models. The only varying factor is the presence and number of residual connections. For each model, we compute the total number of affine regions formed after training using AffineLens. As reported in Table \ref{tab:table3}, empirical results demonstrate that networks equipped with residual connections consistently express a greater number of affine regions compared to those without, under otherwise identical conditions. This experiment further validates the capability of AffineLens to accurately quantify affine regions in PANNs that incorporate residual connections.

\begin{table}[htbp]
\centering
\small
\caption{Number of affine regions expressed by PANNs with and without residual connections, under different epochs.}
\label{tab:table3}
\begin{tabularx}{\linewidth}{l|X X X X X}
\toprule
\diagbox{Model}{Epoch} & 100 & 500 & 1000 & 2000 & 5000 \\
\midrule
\multicolumn{6}{l}{\textbf{2-Res}} \\
\textnormal{With residual}     & 2,317     &   2,940 & 3,920 &   4,472    &   5,449    \\
\textnormal{Without residual}  &   1,701   &   1,929 & 2,477 &   2,776    &    3,518   \\
\cmidrule(lr){1-6}
\multicolumn{6}{l}{\textbf{3-Res}} \\
\textnormal{With residual}     &   3,776     &   5,994   &  7,015  &    8,362     &    9,439     \\
\textnormal{Without residual}  &   2,491     &   3,006   &  3,467  &    4,137     &    4,786     \\
\cmidrule(lr){1-6}
\multicolumn{6}{l}{\textbf{4-Res}} \\
\textnormal{With residual}     &  5,658    &  10,404 & 12,469 &  14,870    &   21,263   \\
\textnormal{Without residual}  &  3,796    &   4,187 & 5,106  &   5,826    &    7,026   \\
\bottomrule
\end{tabularx}
\end{table}

\subsection{Comparison Across Different Activation States}
\label{sec:4.7}
AffineLens enables dynamic visualization of how affine regions and decision boundaries evolve in PANNs across training stages. This experiment illustrates how AffineLens can track geometric and functional changes of a network during optimization, providing insights into how activation functions and training progress influence the formation of affine partitions and the emergence of decision boundaries. Such dynamics are useful for interpreting model behavior, diagnosing training inefficiencies, and guiding architectural design toward desirable expressivity and generalization.

As shown in Fig.~\ref{fig:f7}, we train a six-layer MLP on the two-moons dataset, where each hidden layer has 32 neurons. We compare two variants of the same architecture, one using ReLU and the other using LeakyReLU. For each model, we compute and visualize the affine-region structure at six training stages ($epoch=1,10,50,100,200,500$), and render the corresponding decision boundaries in the input space. Fig.~\ref{fig:f7}(a) shows the evolution of decision boundaries induced by the changing affine partition: blue and red regions denote the current class-associated affine partitions, while the white areas approximate the decision boundary inferred from network outputs. The epoch index and the total number of affine regions at that stage are annotated below each plot (region counts in parentheses). These snapshots reveal how the network progressively refines its approximation of the target distribution and how partition complexity increases as training proceeds.

To further provide geometric intuition beyond 2D projections, Fig.~\ref{fig:f7}(b) visualizes the evolution of affine regions and decision boundaries in a 3D setting. Two families of colored hyperplanes represent class-conditional affine segments, and their intersection surfaces approximate the evolving decision boundary at each training stage. This 3D view highlights how piecewise-affine partitions form and interact during training in higher-dimensional spaces.

\section{Discussion}

AffineLens can be computationally expensive due to region explosion and the heavy reliance on repeated LP solves, which may hinder scalability for very wide/deep networks or high-dimensional $A_0$. It further assumes the network admits a CPA representation, limiting applicability to non-CPA activations or stochastic training-time operators. Future work includes stronger constraint canonicalization and caching/parallel LP pipelines, as well as standardized metrics on a shared $A_0$ (e.g., region-volume statistics or boundary-neighborhood region density).

\section{Conclusion}
This paper presents AffineLens, a principled framework for analyzing the expressive capacity of PANNs through the enumeration and geometric characterization of affine regions. By leveraging connections to hyperplane arrangements, our method enables the precise localization and visualization of affine partitions in high-dimensional input spaces, and remains applicable to architectures with residual connections and arbitrary depth and width. Extensive experiments demonstrate that both width and depth significantly influence region complexity, and that in our experiments increasing width in earlier layers has a strong effect on region count during early training. Furthermore, our analysis suggests that convolutional architectures tend to exhibit more structured and regular projected affine-region patterns in low-dimensional visualizations, compared with MLPs. These findings highlight the utility of AffineLens as a diagnostic tool for dissecting neural network expressivity, offering theoretical insight and practical guidance for architecture design. Future work may extend our methodology to encompass other PANNs.

\newpage

\vskip 0.2in
\IfFileExists{sample.bib}{%
  \bibliography{sample}%

\begin{thebibliography}{99}
\bibitem{missingbib} Bibliography file \texttt{sample.bib} was not found during this compilation.
\end{thebibliography}
}{%
}

\end{document}